\documentclass{article}

\usepackage[preprint]{neurips_2026}

\usepackage[utf8]{inputenc} 
\usepackage[T1]{fontenc}    
\usepackage{hyperref}       
\usepackage{url}            
\usepackage{booktabs}       
\usepackage{amsfonts}       
\usepackage{nicefrac}       
\usepackage{microtype}      
\usepackage{xcolor}         

\usepackage{epigraph}
\usepackage{pifont}
\usepackage{wrapfig}
\usepackage{mathtools}
\usepackage{listings}
\usepackage{tabulary}
\usepackage[table]{xcolor}
\usepackage{tcolorbox} 
\usepackage{amsthm}
\usepackage{booktabs}
\usepackage{caption}
\usepackage{enumitem}
\definecolor{mygrey}{rgb}{0.94,0.96,0.96}

\lstdefinelanguage{json}{
    basicstyle=\ttfamily\small,
    numbers=left,
    numberstyle=\tiny,
    stepnumber=1,
    numbersep=8pt,
    showstringspaces=false,
    breaklines=true,
    frame=single,
    backgroundcolor=\color{gray!10},
    stringstyle=\color{green!50!black},
    keywordstyle=\color{blue},
    commentstyle=\color{gray},
    morestring=[b]",
}

\newtcolorbox{llmoutput}{
  colback=white,
  colframe=black,
  boxrule=0.4pt,
  arc=2pt,
  left=6pt,
  right=6pt,
  top=4pt,
  bottom=4pt
}

\definecolor{myblue}{HTML}{3A7CA5}
\definecolor{myred}{HTML}{A94442}

\newtheorem{innercustomlemma}{Lemma}
\newenvironment{customlemma}[1]
  {\renewcommand\theinnercustomlemma{#1}\innercustomlemma}
  {\endinnercustomlemma}

\newtheorem{lemma}{Lemma}

\DeclareMathOperator*{\supp}{supp}

\title{The Attacker in the Mirror: Breaking Self-Consistency in Safety via Anchored Bipolicy Self-Play}

\author{%
  Gabriele La Malfa$^{1}$\thanks{Equal contribution. Correspondence to \texttt{gabriele.la\_malfa{@}kcl.ac.uk} and \texttt{emanuele.lamalfa{@}cs.ox.ac.uk}.} \quad 
  Emanuele La Malfa$^{2,4}$\textsuperscript{$*$} \quad 
  Saar Cohen$^{2}$ \\
  \textbf{Jie M. Zhang}$^{1}$ \quad 
  \textbf{Michael Luck}$^{3}$ \quad
  \textbf{Michael Wooldridge}$^{2}$ \quad
  \textbf{Elizabeth Black}$^{1}$ \\
  $^{1}$Department of Informatics, King’s College London \\
  $^{2}$Department of Computer Science, University of Oxford \\
  $^{3}$University of Sussex, $^{4}$Institute for Decentralized AI
}

\begin{document}

\maketitle

\vspace{-0.8cm}
\begin{center}
    \includegraphics[height=1em]{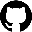} \href{https://github.com/EmanueleLM/AnchoredBipolicySelf-Play}{\texttt{Github}} |  \includegraphics[height=1em]{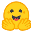} \href{https://huggingface.co/EmanueleLaMalfa/AnchoredBipolicySelf-Play/tree/main}{\texttt{Checkpoints}}
\end{center}

\begin{abstract}
Self-play red team is an established approach to improving AI safety in which different instances of the same model play attacker and defender roles in a zero-sum game, i.e., where the attacker tries to jailbreak the defender; if self-play converges to a Nash equilibrium, the model is guaranteed to respond safely within the settings of the game. Although the parameter sharing enforced by the use of the same model for the two roles improves stability and performance, it introduces fundamental theoretical and architectural limitations. We show that the set of Nash equilibria that can be reached corresponds to a broad class of behaviours that includes trivial always refuse strategies and oracle-like defenders, thus limiting practical applicability. We then show that when attacker and defender share and update the same base model, the dynamics collapse to self-consistency, so that attacks do not enforce adversarial pressure on the defender.
In response, we propose Anchored Bipolicy Self-Play, which trains distinct role-specific LoRA adapters on top of a frozen base model, thereby maintaining stable optimisation while preserving adversarial pressure through explicit role separation. In relation to standard self-play, we show up to 100x greater parameter efficiency than fine-tuning and consistent improvements in safety compared to self-play fine-tuned models. We evaluate on Qwen2.5-\{3B, 7B,14B\}-IT models across widely used safety benchmarks, showing improved robustness without loss of reasoning ability. Cross-play experiments further show that our attacker and defender models are superior to self-play in terms of adversarial defence and safety.

\end{abstract}

\epigraph{\emph{``A slippery defender is a dangerous defender.''}}
{--- Carlo Mazzone (\emph{circa} 1992 CE)}

\vspace{-0.5cm}
\section{Introduction}
While concerns surrounding the safety of NLP models have long been recognised~\cite{arras2016explainingpredictionsnonlinearclassifiers}, Large Language Models (LLMs)~\cite{brown2020languagemodelsfewshotlearners}, given their capabilities and widespread adoption, pose significant security and safety risks that have the potential to erode trust in these systems~\cite{lamalfa2023languagemodelsserviceoverview}, particularly with the rise of LLM agents and agentic applications~\cite{dewitt2025openchallengesmultiagentsecurity}.
Although safeguarding generated outputs is critical for preventing misuse, recent developments reveal ongoing and substantial obstacles to achieving dependable LLM safety~\cite{xu-etal-2024-comprehensive}.
Model patching to prevent jailbreaks constitutes one mitigation strategy~\cite{xiong2025defensivepromptpatchrobust}. More recently, self-play reinforcement learning, which engages multiple models in iterative attack-and-defence cycles, has shown promise in improving model robustness~\cite{liu2025chasingmovingtargetsonline}.

As Figure~\ref{fig:ipsp_architecture} (centre) illustrates, in the Self-Play Red-Teaming (\textsc{Self-RedTeam}) framework by Liu et al.~\cite{liu2025chasingmovingtargetsonline}, a single model is fine-tuned to assume both the attacker role, attempting to jailbreak the defender, and the defender role, aiming to prevent such breaches. The attacker is rewarded for eliciting harmful responses, whilst the defender succeeds by refusing harmful requests or answering benign requests. In this configuration, private reasoning traces are hidden, requiring agents to infer opponent intent exclusively from observable actions.

\begin{figure*}
    \centering
    \includegraphics[width=1\textwidth]{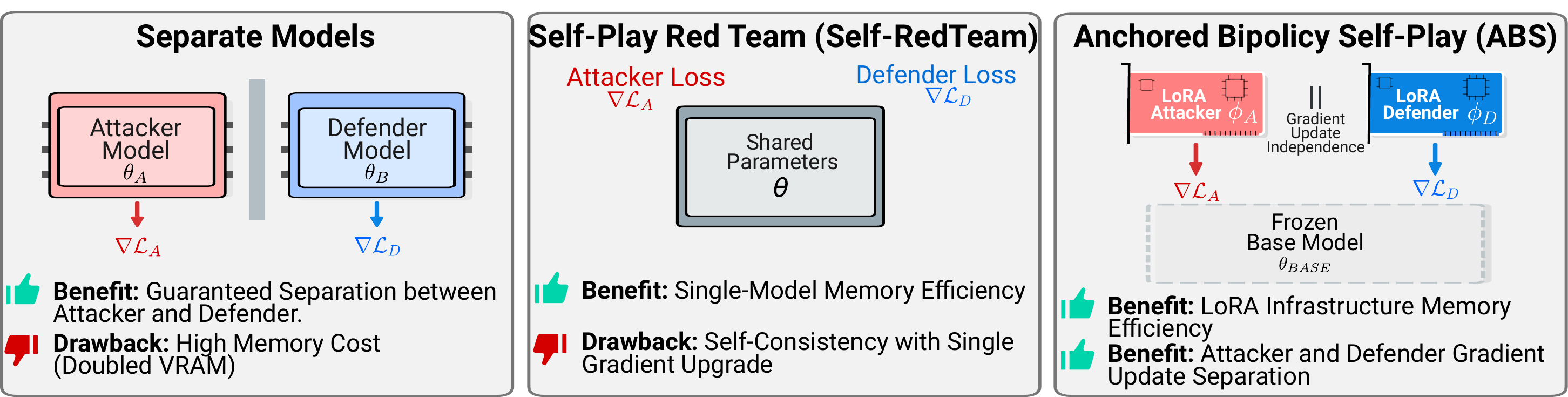}
    \caption{Architectural Comparison of Self-Play Red Team Frameworks.
    \textbf{(Left) Full Model Separation}: Two distinct models guarantee independence and adversarial pressure, but lead to high computational cost. 
    \textbf{(Middle) Single-Model Coupling (\textsc{Self-RedTeam}~\cite{liu2025chasingmovingtargetsonline})}: Attacker and Defender share the same parameters. Gradient updates for the attacker shift the defender's decision boundaries, resulting in self-consistency and weakening the discovery of new vulnerabilities.
    \textbf{(Right) Anchored Bipolicy Self-Play (ABS, Ours)}: We freeze the base model and evolve separate LoRA adapters for the attacker and defender, restoring adversarial pressure with minimal memory cost.}
    \label{fig:ipsp_architecture}
\end{figure*}

Although this approach can be framed as an imperfect-information game, current frameworks are limited by fundamental architectural and theoretical constraints. 
First, we show that relying on convergence to a Nash equilibrium as a safety guarantee admits equilibria that include trivial always-refuse strategies or oracle-like defenders, limiting practical applicability.
Moreover, full parameter sharing, while improving training stability, introduces a strong dependency. In particular, we show that the min-max safety game of Liu et al.~\cite{liu2025chasingmovingtargetsonline} induces a self-consistency bias that potentially reduces the adversarial pressure on the defender.

To re-establish adversarial integrity within the safety game, this work introduces Anchored Bipolicy Self-Play (ABS). The ABS framework enforces strict architectural separation between attacker and defender while maintaining the training stability of self-play (Figure~\ref{fig:ipsp_architecture}, right). By freezing the base model, ABS employs Low-Rank Adaptation (LoRA) to train distinct, lightweight adapter modules for each role, guaranteeing gradient update separation and a substantial reduction of the computational costs compared to the fully fine-tuned state-of-the-art \textsc{Self-RedTeam} and the base model.

ABS matches or outperforms both the base models and the fully fine-tuned \textsc{Self-RedTeam} across multiple safety benchmarks while preserving reasoning and instruction-following abilities. Our method is particularly effective at smaller scales, where it provides significant performance gains over existing self-play frameworks while remaining competitive with fully fine-tuned models at the 14B scale.
We further assess attacker strength in two settings: a cross-play tournament, where ABS attackers and \textsc{Self-RedTeam} defenders of equal size are evaluated against each other, and a transfer setting, where both act as attackers against the base Qwen model. In both cases, ABS attackers outperform \textsc{Self-RedTeam}, achieving an attack success rate that is around $10\%$ higher in cross-play and $15\%$ higher when targeting the same base model.

The contributions of this work are threefold:

\begin{itemize}
    \item We show that the zero-sum formulation of self-play red-teaming allows for a broad class of indistinguishable Nash equilibria, which include both trivial ``always-refuse'' strategies and optimal oracle defenders. This result provides a theoretical explanation for why convergence alone does not ensure model utility.
    \item We identify and geometrically formalise the ``self-consistency'' phenomenon inherent in shared-parameter architectures. Our analysis demonstrates that coupled gradient updates suppress adversarial pressure and hinder the exploration of worst-case prompts.
    \item We propose Anchored Bipolicy Self-Play (ABS), a parameter-efficient framework that employs role-specific LoRA adapters to maintain architectural separation. This approach restores adversarial integrity by enabling independent evolution of the attacker and defender, while also substantially reducing the computational cost associated with self-play training.
\end{itemize}

The paper is organised as follows: we first review the relevant literature in reinforcement learning for safety, self-play, and parameter-efficient training. We then show how the analysis of safety via Nash equilibria is incomplete, and that the current formulation of \textsc{Self-RedTeam} reduces to self-consistency. We then introduce our technique (ABS), and conduct the experiments to prove its performance and scalability. All the results are available \href{https://github.com/EmanueleLM/AnchoredBipolicySelf-Play}{[\textcolor{myblue}{here}]}.

\section{Related Work}

\textbf{Safety, Jailbreaking, and Red-Teaming.}
Standard red team training relies on human feedback to identify harmful responses and reinforce model behaviour accordingly~\cite{bai2022traininghelpfulharmlessassistant, dai2023saferlhfsafereinforcement}. Owing to the limited coverage of this approach, models are often encouraged to explore novel attack prompts~\cite{perez2022redteaminglanguagemodels}. For instance, Mehrotra et al.~\cite{mehrotra2024treeattacksjailbreakingblackbox} propose a pruning strategy to pre-select attacks with higher jailbreak potential. Hong et al.~\cite{hong2024curiositydrivenredteaminglargelanguage} instead optimise jointly for reward and novelty, incentivising the generation of both successful and previously unseen prompts in a curiosity-driven setting.
Other lines of work focus on co-evolving attacker and defender policies to avoid static adversaries~\cite{jain2023baselinedefensesadversarialattacks}. A central challenge in this setting is computational cost~\cite{howe2025scalingtrendslanguagemodel}. Xhonneux et al.~\cite{xhonneux2024efficientadversarialtrainingllms} address this by operating in the continuous embedding space, reducing the expense associated with discrete token-level attacks. Liu et al.~\cite{liu2025chasingmovingtargetsonline} introduce an online self-play framework in which a single model instantiates both attacker and defender, enabling mutual adaptation. More recently, Wen et al.~\cite{wen2026magiccoevolvingattackerdefenderadversarial} propose MAGIC, a decoupled co-evolving formulation designed to mitigate optimisation conflicts through full parameter updates. Tan et al.~\cite{tan2026triplayrltriroleselfplayreinforcement} extend this line of work with a tri-role framework that incorporates an additional evaluator to stabilise training via multi-agent interaction.

\textbf{Parameter-Efficient Fine-Tuning Techniques.}
Parameter-Efficient Fine-Tuning (PEFT) techniques have demonstrated effectiveness in model training while offering significant computational efficiency compared to full model fine-tuning. Among these methods, LoRA~\cite{hu2021loralowrankadaptationlarge} applies a low-rank update to the weights while keeping the base model parameters fixed. For instance, in transformer architectures that underpin modern large language models (LLMs), LoRA modifies the attention weights while leaving the multi-layer perceptron weights unchanged. 
Further variations of LoRA work on either reducing the computational learning costs, such as QLoRA~\cite{dettmers2023qloraefficientfinetuningquantized}, or dynamically modifying the dimension of the weight matrices (rank) to improve critical learning momenta, like AdaLoRA~\cite{zhang2023adaloraadaptivebudgetallocation}. SaLoRA~\cite{li2025salorasafetyalignmentpreservedlowrank} instead proposes a method to improve safety alignment of LLMs. Finally, in line with this paper, Red-Bandit implements a LoRA-based post-training to reduce LLMs' vulnerabilities to attacks~\cite{ziakas2025redbandittesttimeadaptationllm}.     

\textbf{Summary and Research Gap.}
\textsc{Self-RedTeam}~\cite{liu2025chasingmovingtargetsonline} serves as the primary benchmark as it established the foundational setting for online symmetric self-play in LLMs’ safety, where a single model alternates between attacker and defender roles. Recent literature has shifted toward more complex co-evolutionary games, such as MAGIC~\cite{wen2026magiccoevolvingattackerdefenderadversarial}, which proposes role decoupling through full-parameter updates, and TriPlay-RL~\cite{tan2026triplayrltriroleselfplayreinforcement}, which introduces an auxiliary ``Evaluator'' role to stabilise co-evolution. Unlike our work, they do not provide a theoretical treatment of the gradient collision and self-consistency bias inherent in shared-parameter architectures. ABS fills this gap by introducing a parameter-efficient method that restores adversarial integrity through architectural separation.

\section{Nash Equilibria, Model Safety, and Self-Play}
\textsc{Self-RedTeam} by Liu et al.~\cite{liu2025chasingmovingtargetsonline} proposes model safety via self-play, wherein a single base model jointly instantiates an attacker $A$ (issuing jailbreak prompts) and a defender $D$ (responding). Training proceeds online in a sequential zero-sum game, where a Nash equilibrium corresponds to safety (see \S\ref{sec:Nash Equilibria in Self-Play Red-Teaming}). Despite growing uptake~\cite{wen2026magiccoevolvingattackerdefenderadversarial} and extensions to general-sum settings~\cite{liao2024efficacy}, we show that (i) the \textsc{Self-RedTeam} formulation admits indistinguishable equilibria spanning trivial always-refusal and oracle defenders (see \S\ref{sec:Analysis of Extremal Nash Equilibria}); and (ii) the objective is more naturally cast as single-model self-consistency, with implications for training dynamics and equilibria (see \S\ref{sec:Min-Max Constraint Conflict and Self-Consistency}).

\subsection{Nash Equilibria in \textsc{Self-RedTeam}}
\label{sec:Nash Equilibria in Self-Play Red-Teaming}
Let $\mathcal{Y}_A$ be the set of attacker prompts and $\mathcal{Y}_D$ be the set of defender responses.
For each player $P \in \{A,D\}$, we define the probability simplex over $\mathcal{Y}_P$ as $\Delta(\mathcal{Y}_P)$. In the two-player game formulation of red-teaming, the attacker first proposes a prompt $y_A$ sampled from a mixed strategy $\pi_A \in \Delta(\mathcal{Y}_A)$, referred to as a policy. Given the prompt $y_A$, the defender then generates a response $y_D$ sampled from a conditional strategy $\pi_D(\cdot \mid y_A)$ with $\pi_D: \mathcal{Y}_A  \rightarrow \Delta(\mathcal{Y}_D)$.\footnote{Without loss of generality, we omit that in \textsc{Self-RedTeam}, the attack is initially conditioned on a \texttt{seed}, while the attacker's and defender's policies are further conditioned on the instructions to generate the attack and the defence, respectively.}
Afterwards, a reward model $r:\mathcal{Y}_A \times \mathcal{Y}_D \to [-1,1]$ evaluates each prompt-response pair, i.e., for any such pair $(y_A,y_D)$, the defender's reward is $r(y_A,y_D)$, and the attacker receives the opposite payoff $-r(y_A,y_D)$.

For an attacker and a defender that aim to maximise their rewards, this setting corresponds to a zero-sum sequential game where (i) unsafe responses are strictly penalised, i.e., $r(y_A,y_D) = -1$ if $(y_A,y_D)$ is unsafe; (ii) safe responses are rewarded non-negatively, i.e., $r(y_A,y_D) \ge 0$ if $(y_A,y_D)$ is safe, and (iii) there exists a refusal response $y^{\mathrm{ref}}\in\mathcal{Y}_D$ such that $r(y_A,y^{\mathrm{ref}})=0$ for all $y_A\in\mathcal{Y}_A$.

Under these assumptions, the expected reward is $V(\pi_A,\pi_D)=\mathbb{E}_{y_A\sim\pi_A,\;y_D\sim\pi_D(\cdot\mid y_A)} \big[ r(y_A,y_D) \big]$ and the game admits the following min-max representation:
\[
\min_{\pi_A}\max_{\pi_D} V(\pi_A,\pi_D).
\]

A Nash equilibrium $(\pi_A^\ast,\pi_D^\ast)$ is thus a pair of strategies such that neither player has an incentive to deviate, i.e., $V(\pi_A^\ast,\pi_D^\ast) \ge V(\pi_A^\ast,\pi_D)$ for all $\pi_D$, and $V(\pi_A^\ast,\pi_D^\ast) \le V(\pi_A,\pi_D^\ast)$ for all $\pi_A$. When the two players' policies converge to a Nash equilibrium $(\pi_A^\ast,\pi_D^\ast)$, Liu et al. \cite{liu2025chasingmovingtargetsonline} proved that the defender guarantees a safe response to any attacker prompt, i.e., $r(y_A, \pi_D^\ast(y_A)) \ge 0$ for any $y_A$.

We show that, under the above assumptions, extremal Nash equilibria attaining the minimum or maximum defender expected payoff are problematic in both theory and practice (see \S\ref{sec:Analysis of Extremal Nash Equilibria}). On the one hand, equilibria that are qualitatively distinct yet indistinguishable to the learning dynamics offer limited practical guidance. On the other hand, when a single model instantiates both attacker and defender, the dynamics can induce implicit coordination or collusion between roles (see \S\ref{sec:Min-Max Constraint Conflict and Self-Consistency}).

\subsection{Analysis of Extremal Nash Equilibria}
\label{sec:Analysis of Extremal Nash Equilibria}
In this section, we derive a ``lower bound'' extremal that corresponds to the case where the defender always refuses, guaranteeing safety but compromising usefulness. Complementarily, we characterise the ``upper bound'' extremal that corresponds to an oracle-like defender that always provides safe and maximally helpful responses.

\paragraph{Bound 1: An always-refuse equilibrium.}
Recall our assumption that the defender can always issue a refusal response $y^{\mathrm{ref}}\in\mathcal{Y}_D$ such that $r(y_A,y^{\mathrm{ref}})=0$ for all $y_A\in\mathcal{Y}_A$. Thus, for any prompt $y_A$, we define the always-refuse defender strategy as $\pi_D^{\mathrm{ref}}(y_D \mid y_A) = \mathbf{1}\{y_D = y^{\mathrm{ref}}\}$, which yields an expected reward $V(\pi_A, \pi_D^{\mathrm{ref}}) = 0$ for any attacker policy $\pi_A$. 

A pair $(\pi_A^\ast, \pi_D^{\mathrm{ref}})$ is a Nash equilibrium if the attacker policy $\pi_A^\ast$ is sufficiently adversarial such that no defender policy achieves a strictly positive expected payoff, i.e., $\max_{\pi_D} V(\pi_A^\ast, \pi_D)=0$. We therefore focus on the attacker policies in a local neighbourhood of a reference adversarial policy $\bar{\pi}_A$, which concentrates on prompts where any non-refusal response yields no positive reward, and where nearby policies preserve this property (see Appendix~\ref{appendix:TV inequality}). This reflects the intuition that attacks may adapt while remaining close to a reference distribution, ensuring semantic consistency.

Within this neighbourhood, we prove that deviations in the attacker policy only induce controlled changes in the defender’s payoff (Lemma \ref{lemma:TV inequality}). We thus focus on attacker policies $\pi_A^\ast$ that preserve the adversarial property of $\bar{\pi}_A$ locally, in the sense that they do not admit any defender strategy with strictly positive expected payoff, i.e., we require $\mathbb{E}_{y_A \sim \pi_A^\ast, y_D \sim \pi_D'}[r(y_A, y_D)] \le 0$ for any alternative defender strategy $\pi_D'$.
This condition can be interpreted as a \textit{worst-case} constraint on the attacker: within the local neighbourhood, any attempt by the defender to provide a helpful response yields no positive gain and may incur risk, so the defender has no incentive to deviate from $y^{\mathrm{ref}}$. Specifically, providing a helpful response risks a payoff of $-1$ that is not offset by safe, high-reward completions, making the expected value of answering non-positive. Hence, $\pi_D^{\mathrm{ref}}$ is a best response to $\pi_A^\ast$. Further, as the attacker cannot reduce the payoff below $0$ against a refusing defender, $\pi_A^\ast$ is a best response.

\paragraph{Bound 2: The defender acts like an oracle.}
Let $\mathcal{M} \subseteq \mathcal{Y}_A$ be the set of malicious prompts, for which the maximum achievable safe reward is zero. That is, $\max_{y_D} r(y_A,y_D)=0$ for any $y_A \in \mathcal{M}$ and $\max_{y_D} r (y_A,y_D)>0$ for any $y_A \notin \mathcal{M}$. Assume there exists an oracle defender policy $\pi_D^{\mathrm{rec}}$ that is pointwise optimal: for every $y_A \in \mathcal{Y}_A$, it picks a response $y_D$ to maximize $r(y_A, y_D)$. Namely, for all $y_A \in \mathcal{M}$, $\pi_D^{\mathrm{rec}}$ ensures safety by generating a response $y_D^{\mathrm{safe}}$ with $r(y_A, y_D^{\mathrm{safe}}) = 0$, and for $y_A \notin \mathcal{M}$, it provides a maximally helpful response $y_D$ with $r(y_A, y_D) > 0$.

Therefore, a pair $(\pi_A^\ast, \pi_D^{\mathrm{rec}})$ is a Nash equilibrium if the attacker's policy $\pi_A^\ast$ minimises the expected reward against the oracle. Formally, $\pi_A^\ast$ must satisfy:
\begin{equation*}
    \pi_A^\ast \in \arg\min_{\pi_A} \mathbb{E}_{y_A \sim \pi_A} \left[ \max_{\pi_D} \mathbb{E}_{y_D \sim \pi_D(\cdot|y_A)}\left[r(y_A, y_D)\right] \right] = \arg\min_{\pi_A} \mathbb{E}_{y_A \sim \pi_A} \left[ \max_{y_D} r(y_A, y_D) \right].
\end{equation*}
where the last equality follows from $\max_{\pi_D} \mathbb{E}_{y_D \sim \pi_D(\cdot|y_A)}\left[r(y_A, y_D)\right]=\max_{y_D} r(y_A, y_D)$ since the expected reward is linear in the mixed strategy $\pi_D(\cdot|y_A) \in \Delta(\mathcal{Y}_D)$. Thus, the maximum over the simplex $\Delta(\mathcal{Y}_D)$ it is attained at a vertex, meaning that the optimal \textit{mixed} strategy simply places all probability mass on the best \textit{pure} strategy $y_D$ maximizing $r(y_A,y_D)$.

As $\max_{y_D} r(y_A, y_D) = 0$ for all $y_A \in \mathcal{M}$, and $\max_{y_D} r (y_A,y_D)>0$ for any $y_A \notin \mathcal{M}$, it follows that any equilibrium attacker policy $\pi_A^\ast$ must be supported on $\mathcal{M}$, i.e., $\supp(\pi_A^\ast) \subseteq \mathcal{M}$, where $\supp(\pi_A^\ast) := \{y_A \in \mathcal{Y}_A : \pi_A^\ast(y_A) >0\}$ denotes the support of $\pi_A^\ast$. At such an equilibrium, the attacker cannot further reduce the payoff, as safety is guaranteed, while the defender cannot improve its payoff on malicious inputs without violating the safety constraint and incurring a $-1$ penalty.

\begin{tcolorbox}[
    colback=myblue!5!white,
    colframe=myblue,
    arc=1mm,
    boxrule=1pt,
    title=\textbf{Nash equilibria and Safety Guarantees},
]
The extremal Nash equilibria with respect to the defender's payoff are:
\begin{itemize}[noitemsep,topsep=1pt,parsep=1pt,partopsep=0pt]
    \item A degenerate \emph{always-refuse} equilibrium (safe but unhelpful);
    \item An \emph{oracle} equilibrium (safe and useful), when such a policy exists.
\end{itemize}

\vspace{1mm}
\textbf{Conclusion:} The theoretical framework admits safety-preserving equilibria; however, such equilibria range from a trivial \emph{always-refuse} solution to an \textit{oracle defender}.
\end{tcolorbox}

\subsection{Min-Max Constraint Conflict and Self-Consistency}
\label{sec:Min-Max Constraint Conflict and Self-Consistency}
Assume now that the attacker and defender share parameters, i.e., there exists a single policy $\pi_\theta$ such that $\pi_A = \pi_D = \pi_\theta$.
Under this constraint, the optimisation problem involves a single parameter vector $\theta$ balancing two conflicting objectives:
\begin{enumerate}[noitemsep,topsep=1pt,parsep=1pt,partopsep=0pt]
    \item \textbf{Defender Role:} Maximize $V$ (generate safe responses).
    \item \textbf{Attacker Role:} Minimize $V$ (generate successful jailbreaks).
\end{enumerate}
Since the parameters are shared, the attacker cannot adapt independently without simultaneously altering the defender. In the zero-sum red-teaming game, the attacker’s and defender’s rewards for any prompt-response pair $(y_A,y_D)$ are $R_A(y_A,y_D):=-r(y_A,y_D)$ and $R_D(y_A,y_D):=r(y_A,y_D)$, respectively. We denote the defender’s and attacker’s objectives as follows (respectively):
\[
J_D(\theta) \coloneqq \mathbb{E}_{y_A \sim \pi_\theta,\; y_D \sim \pi_\theta(\cdot \mid y_A)}[r(y_A,y_D)],
\quad
J_A(\theta) \coloneqq \mathbb{E}_{y_A \sim \pi_\theta,\; y_D \sim \pi_\theta(\cdot \mid y_A)}[-r(y_A,y_D)]
\]
The resulting policy-gradient update is given by the sum of the defender and attacker gradients:
\begin{equation}\label{GradientOpposingForces}
\begin{aligned}
    \nabla_\theta J_D(\theta) + \nabla_\theta J_A(\theta)
    &=
    \mathbb{E}_{y_A,y_D}
    \Big[
    r(y_A,y_D)
    \big(
    \underbrace{\nabla_\theta \log \pi_\theta(y_D \mid y_A)}_{\text{Force 1: Improve Defence}}
    -
    \underbrace{\nabla_\theta \log \pi_\theta(y_A)}_{\text{Force 2: Improve Attacks}}
    \big)
    \Big].
\end{aligned}
\end{equation}
To make the coupling explicit, let $\eta > 0$ denote the learning rate and consider a defender-only step, i.e., $\theta^+ = \theta + \eta\, \nabla_\theta J_D(\theta)$. Then, a first-order Taylor expansion of the attacker objective yields
\begin{equation}\label{eq:interference}
    J_A(\theta^+) = J_A(\theta) + \eta \left\langle \nabla_\theta J_A(\theta),\, \nabla_\theta J_D(\theta) \right\rangle + O(\eta^2).
\end{equation}

Thus, at first order, the effect of a defender update is governed by the inner product between the attacker and defender gradients. When these directions are aligned, improving the defender simultaneously increases the attacker’s objective, and vice versa. This coupling highlights how shared-parameter training can induce self-consistency instead of fully robust adversarial behaviour.

Geometrically, we can observe this by analysing the signs of the rewards in Eq.~\eqref{GradientOpposingForces}:

\begin{itemize}[noitemsep,topsep=1pt,parsep=1pt,partopsep=0pt]
    \item 
    \textbf{Case 1: Resistance ($r(y_A,y_D) > 0$):} When the defender resists an attack, Force 1 increases the likelihood of the safe response $y_D$ while Force 2 decreases the likelihood of the attack prompt $y_A$. This drives $\pi_\theta$ to avoid prompts it already defends against, leading to self-consistency rather than the discovery of fundamentally stronger attacks.
    
    \item \textbf{Case 2: Success ($r(y_A,y_D) < 0$):} Conversely, if a jailbreak succeeds, Force 2 attempts to reinforce the successful attack $y_A$. Simultaneously, the heavy penalty on the first term (Force 1) tries to update $\theta$ to suppress the vulnerability that permitted the jailbreak. 
\end{itemize}

Because these updates are applied simultaneously to the shared parameter vector $\theta$, the gradients effectively collide. Rather than sharpening the attacker and strengthening the defender, the model converges to a stagnant middle ground. This prevents exploration of the adversarial frontier, yielding a ``muddy'' representational space in which both roles are compromised.

\begin{tcolorbox}[
    colback=myblue!5!white,
    colframe=myblue,
    arc=1mm,
    boxrule=1pt,
    title=\textbf{Self-consistency and Safety Guarantees},
]
In summary:
\begin{itemize}[noitemsep,topsep=1pt,parsep=1pt,partopsep=0pt]
\item Self-play Nash equilibria can reduce to \emph{self-consistency} equilibria.
\item As an optimisation process, \textsc{Self-RedTeam}~\cite{liu2025chasingmovingtargetsonline} induces coupled gradients, where improving the defender can weaken the attacker.
\end{itemize}

\vspace{1mm}
\textbf{Conclusion:} A well-posed adversarial saddle-point formulation requires distinct policies.
\end{tcolorbox}

\section{Anchored Bipolicy Self-Play}
We now formalise the Anchored Bipolicy Self-Play (ABS) policy for the attacker and the defender: we rely on LoRA modules to ensure attacker-defender separation so that the adversarial pressure on the defender remains unaltered, while maintaining the objective function defined in \textsc{Self-RedTeam}.

\subsection{Role-conditioned LoRA Policies}
\label{sec:Role-conditioned LoRA Policies}
Let $\pi_\theta$ denote a single policy network with a shared frozen backbone and role-specific LoRA adapters. The effective parameters are given by
$\Theta = (\theta, \phi_A, \phi_D)$, where $\theta$ denotes the frozen backbone parameters, and $\phi_A, \phi_D$ are the LoRA parameters for the attacker and defender, respectively. For a Transformer-based LLM model~\cite{vaswani2023attentionneed}, and a subset of backbone parameters $\theta^{l} \subseteq \theta$ (e.g., the query and value projections of self-attention), a LoRA module parameterised by $\phi^{l} \subseteq \phi$ defines a low-rank update added to $\theta^{l}$ and trained while keeping $\theta^{l}$ frozen. Then, actions for the attacker and the defender are sampled as $y_A \sim \pi_{\theta, \phi_A}(\cdot)$ and $y_D \sim \pi_{\theta, \phi_D}(\cdot \mid y_A)$.

ABS adapts the training method in~\cite{liu2025chasingmovingtargetsonline}, which leverages Re++~\cite{hu2025reinforcestabilizingcriticfreepolicy}, a lightweight reinforcement learning algorithm extending PPO and REINFORCE~\cite{williams1992simple,schulman2017proximalpolicyoptimizationalgorithms}. It uses token-level KL regularisation to penalise deviations from a reference policy, along with different learning-rate warm-up procedures for the attacker and defender. Implementation details are provided in the next section and Appendix~\ref{a:ABS-modules}.
\newline
The attacker and defender are then trained using separate objectives:
\begin{equation}
\begin{aligned}
J_A(\phi_A|\phi_D) 
= 
\mathbb{E}_{y_A \sim \pi_{\theta,\phi_A},\, y_D \sim \pi_{\theta,\phi_D}(\cdot \mid y_A)} 
\big[ R_A(y_A,y_D) \big], \\
J_D(\phi_D|\phi_A) 
= 
\mathbb{E}_{y_A \sim \pi_{\theta,\phi_A},\, y_D \sim \pi_{\theta,\phi_D}(\cdot \mid y_A)} 
\big[ R_D(y_A,y_D) \big].
\end{aligned}
\end{equation}
The opposing player's parameters are fixed when computing gradients, with policy-gradient updates:
\begin{equation}
\label{eq:separate updates}
\begin{aligned}
    \nabla_{\phi_A} J_A
    &=
    \mathbb{E}\!\left[
    \nabla_{\phi_A} \log \pi_{\theta, \phi_A}(y_A)\,R_A(y_A,y_D)
    \right], \\
    \nabla_{\phi_D} J_D
    &=
    \mathbb{E}\!\left[
    \nabla_{\phi_D} \log \pi_{\theta, \phi_D}(y_D \mid y_A)\,R_D(y_A,y_D)
    \right].
\end{aligned}
\end{equation}
These updates evolve separately and restrict optimisation to low-rank subspaces over a shared frozen backbone. This parameter separation restores the distinction between attacker and defender updates, eliminating the parameter-level interference observed in Eq.~\eqref{eq:interference}; 
To formalise this, consider a defender-only update step $\phi_D^+ = \phi_D + \eta\, \nabla_{\phi_D} J_D(\phi_A,\phi_D)$, where $\eta > 0$ is the learning rate. A first-order Taylor expansion of the defender's objective yields
\begin{align}
J_D(\phi_A,\phi_D^{+})
=
J_D(\phi_A,\phi_D)
+
\eta \|\nabla_{\phi_D} J_D(\phi_A,\phi_D)\|^2
+
O(\eta^2).
\end{align}

Importantly, this dependence is purely distributional: the defender parameters update independently of the attacker parameters, i.e., $\frac{\partial \phi_D}{\partial \phi_A} = 0.$
We comprehensively formalise this in Appendix \ref{sec:Absence of Parameter-Level Coupling}.

\section{Experiments \& Results}
To compare our ABS to the fully fine-tuned \textsc{Self-RedTeam} models by Liu et al.~\cite{liu2025chasingmovingtargetsonline}, we adapted their code to train two independent LoRA modules, one for the attacker and one for the defender.
In their setting, the attacker is prompted with a seed from the training partition of WildJailBreak~\cite{jiang2024wildteamingscaleinthewildjailbreaks} (see their Appendix C.4.1), with a 50:50 ratio of harmful and benign prompts.
The prompt instructs the attacker and defender to first think, wrapped by \texttt{<think></think>} tokens, and then provide the answer between \texttt{<answer></answer>} tags. The defender produces its response in a similar format; we conducted the evaluation on the pairs of attacks and defences generated at each iteration.
As in \textsc{Self-RedTeam}, we use WildGuard~\cite{han2024wildguardopenonestopmoderation} as the reward model, which classifies the prompt and response harmfulness, and whether the response is a refusal to answer the attack.

We trained Qwen2.5-\{3B, 7B 14B\}-IT models~\cite{qwen2025qwen25technicalreport} with different LoRA settings, varying the rank and alpha values. The rank defines the hidden dimension of the LoRA representation and thus the number of trainable parameters, while alpha scales the influence of the adapter w.r.t. the base model. Table~\ref{tab:ABS-modules} and Appendix~\ref{a:ABS-modules} report more details of the setting.
Our training is entirely conducted on three NVIDIA H200 SXM GPUs with 141GB of vRAM.

For the safety evaluation, we employed eight tests in the Ai2 Safety Toolkit~\cite{jiang2024wildteamingscaleinthewildjailbreaks}: six tests concern harmful refusal, i.e., if a model responds to malicious prompts; two benign tests benchmark the ability of a model to comply with non-harmful instructions. The test suite measures the Attack Success Rate (ASR) and Robustness to Attacks (RTA).
We also benchmark ABS and the \textsc{Self-RedTeam} Qwen by Liu et al.~\cite{liu2025chasingmovingtargetsonline} on five standard reasoning and math benchmarks (GSM8K~\cite{cobbe2021trainingverifierssolvemath}, CommonSense~\cite{talmor2019commonsenseqaquestionansweringchallenge}, BoolQ~\cite{clark2019boolqexploringsurprisingdifficulty}, ARC Challenge~\cite{clark2018thinksolvedquestionanswering}, and HellaSwag~\cite{zellers2019hellaswagmachinereallyfinish}) to guarantee that they do not lose their base capabilities. 
We ran our tests on a single NVIDIA H200 SXM GPU with 141GB of vRAM and report the accuracy (pass@1) of each model. 

\begin{table*}[b]
\centering
\resizebox{\textwidth}{!}{
\begin{tabular}{l | cc | c | c | c | c | c || c | c || c}
\toprule
& \multicolumn{7}{c||}{\textbf{Harmful Refusal}} 
& \multicolumn{2}{c||}{\textbf{Benign Compliance}} 
& \textbf{Benchmark} \\
\cline{2-11}
& \multicolumn{2}{c|}{\textbf{WG:Test}} 
& \textbf{WJB} 
& \textbf{DAN} 
& \textbf{HarmBench} 
& \textbf{XSTest} 
& \textbf{StrongREJECT} 
& \textbf{WJB} 
& \textbf{XSTest} 
& \textbf{Reasoning/Math} \\

& \textbf{adv} 
& \textbf{vani}
& \textbf{adv} 
& \textbf{adv} 
& \textbf{adv} 
& \textbf{vani}
& \textbf{adv} 
& \textbf{vani}
& \textbf{vani}
& \textbf{-} \\

\textbf{Method} 
& \textbf{ASR} $\downarrow$ 
& \textbf{ASR} $\downarrow$ 
& \textbf{ASR} $\downarrow$ 
& \textbf{ASR} $\downarrow$ 
& \textbf{ASR} $\downarrow$ 
& \textbf{RTA} $\uparrow$
& \textbf{ASR} $\uparrow$
& \textbf{Comply} $\uparrow$
& \textbf{LC Winrate} $\uparrow$
& \textbf{Avg. Accuracy} $\uparrow$ \\

\midrule
Qwen2.5-3B-IT & 0.282 & 0.119 & 0.949 & 0.677 & 0.241 & 0.610 & 0.497 & 0.992 & 0.960 & $0.745 \pm 0.042$ \\
\quad + \textsc{Self-RedTeam} & 0.234 & 0.056 & 0.885 & 0.603 & 0.156 & 0.655 & 0.559 & \textbf{1.000} & 0.932 & $0.671 \pm 0.095$ \\
\quad + \cellcolor{mygrey}\textsc{ABS} & \cellcolor{mygrey}\textbf{0.122} & \cellcolor{mygrey}\textbf{0.019} & \cellcolor{mygrey}\textbf{0.721} & \cellcolor{mygrey}\textbf{0.320} & \cellcolor{mygrey}\textbf{0.088} & \cellcolor{mygrey}\textbf{0.680} & \cellcolor{mygrey}\textbf{0.669} & \cellcolor{mygrey}0.984 & \cellcolor{mygrey}\textbf{0.952} & \cellcolor{mygrey}$0.738 \pm 0.047$ \\
\midrule
\it Improv. \textsc{Self-RedTeam} (\%) & \color{myblue}{+17.0} & \color{myblue}{+52.9} & \color{myblue}{+6.7} & \color{myblue}{+10.9} & \color{myblue}{+35.3} & \color{myblue}{+7.4} & \color{myblue}{+12.5} & \color{myblue}{+0.8} & \color{myred}{-2.9} & \color{myred}{-7.4} \\
\it Improv. ABS (\%) & \color{myblue}{+56.7} & \color{myblue}{+84.0} & \color{myblue}{+24.0} & \color{myblue}{+52.7} & \color{myblue}{+63.5} & \color{myblue}{+11.5} & \color{myblue}{+34.6} & \color{myred}{-0.8} & \color{myred}{-0.8} & \color{myred}{-0.7} \\
\bottomrule
\\
Qwen2.5-7B-IT & 0.279 & 0.040 & 0.897 & 0.613 & 0.234 & 0.915 & 0.470 & \textbf{0.992} & \textbf{0.912} & $0.850 \pm 0.026$ \\
\quad + \textsc{Self-RedTeam} & 0.178 & \textbf{0.002} & \textbf{0.726} & 0.423 & \textbf{0.156} & \textbf{0.925} & 0.371 & \textbf{0.992} & 0.888 & $0.818 \pm 0.037$ \\
\quad + \cellcolor{mygrey}\textsc{ABS} & \cellcolor{mygrey}\textbf{0.169} & \cellcolor{mygrey}0.010 & \cellcolor{mygrey}0.756 & \cellcolor{mygrey}\textbf{0.383} & \cellcolor{mygrey}\textbf{0.156} & \cellcolor{mygrey}0.915 & \cellcolor{mygrey}\textbf{0.644} & \cellcolor{mygrey}0.988 & \cellcolor{mygrey}0.884 & \cellcolor{mygrey}$0.847 \pm 0.024$ \\
\midrule
\it Improv. \textsc{Self-RedTeam} (\%) 
& \color{myblue}{+36.18} 
& \color{myblue}{+94.19} 
& \color{myblue}{+19.12} 
& \color{myblue}{+30.98} 
& \color{myblue}{+33.36} 
& \color{myblue}{+1.09} 
& \color{myred}{-21.24} 
& 0.0 & \color{myred}{-2.63} 
& \color{myred}{-3.2} \\

\it Improv. ABS (\%) 
& \color{myblue}{+39.37} 
& \color{myblue}{+76.51} 
& \color{myblue}{+15.77} 
& \color{myblue}{+37.50} 
& \color{myblue}{+33.36} 
& 0.00 
& \color{myblue}{+36.97} 
& \color{myred}{-0.6} 
& \color{myred}{-3.07} 
& \color{myred}{-0.3} \\
\bottomrule
\\
Qwen2.5-14B-IT & 0.146 & 0.017 & 0.845 & 0.647 & 0.119 & 0.735 & 0.647 & \textbf{1.000} & \textbf{0.984} & $0.874 \pm 0.030$ \\
\quad + \textsc{Self-RedTeam} & \textbf{0.116} & \textbf{0.005} & \textbf{0.661} & 0.467 & \textbf{0.056} & 0.775 & 0.701 & 0.996 & 0.968 & $0.878 \pm 0.035$ \\
\quad + \cellcolor{mygrey}\textsc{ABS} & \cellcolor{mygrey}0.131 & \cellcolor{mygrey}0.022 & \cellcolor{mygrey}0.706 & \cellcolor{mygrey}\textbf{0.446} & \cellcolor{mygrey}0.093 & \cellcolor{mygrey}\textbf{0.770} & \cellcolor{mygrey}\textbf{0.748} & \cellcolor{mygrey}\textbf{1.000} & \cellcolor{mygrey}0.968 & \cellcolor{mygrey}$0.874 \pm 0.030$ \\
\midrule
\it Improv. \textsc{Self-RedTeam} (\%) 
& \color{myblue}{+20.55} 
& \color{myblue}{+70.59} 
& \color{myblue}{+21.77} 
& \color{myblue}{+27.82} 
& \color{myblue}{+52.94} 
& \color{myblue}{+5.44} 
& \color{myblue}{+8.35} 
& \color{myred}{-0.4} 
& \color{myred}{-1.63} 
& \color{myblue}{+0.4} \\

\it Improv. ABS (\%)
& \color{myblue}{+10.27}
& \color{myred}{-29.41}
& \color{myblue}{+16.44}
& \color{myblue}{+31.07}
& \color{myblue}{+21.85}
& \color{myblue}{+4.76}
& \color{myblue}{+15.61}
& 0.0
& \color{myred}{-1.63}
& 0.0 \\
\bottomrule

\end{tabular}
}
\caption{Performance of Qwen2.5-\{3B, 7B, 14B\}-IT, \textsc{Self-RedTeam} fine-tuned, and our \colorbox{mygrey}{ABS} models on six benchmarks for harmful refusal, two for benign compliance, and five (average) for instruction following and reasoning. Improvement is calculated relative to the base model. {\color{myblue}{Blue}} indicates improvements (lower ASR or higher RTA/Compliance), {\color{myred}{red}} indicates regression.}
\label{tab:comparative_performance}
\end{table*}

\subsection{ABS Safety and Efficiency Evaluation}

As Table~\ref{tab:comparative_performance} reports, our method outperforms the corresponding Qwen-IT base model and, in most cases, the fully fine-tuned \textsc{Self-RedTeam} model by Liu et al.~\cite{liu2025chasingmovingtargetsonline}.
On standard reasoning benchmarks, our models do not degrade their performance, as shown in the last column of Table~\ref{tab:comparative_performance} (full table in Appendix, Table~\ref{tab:performance-safety-qwen}), while the fine-tuned \textsc{Self-RedTeam} Qwen2.5-3B-IT experiences relevant drops, revealing that fine-tuning small models for safety may incur catastrophic forgetting.
In terms of efficiency, ABS introduces, for Qwen2.5-14B-IT with rank $32$ (our largest module), an additional $250M$ parameters ($2\%$ of the total parameters), and completes a full training run, on three H200 GPUS, in five hours. The details of other implementations are reported in Table~\ref{tab:ABS-modules}.

We also observe that in some cases LoRA becomes quickly unstable, and derails to produce nonsense output; a typical collapse mode consists of Qwen producing a sequence of \texttt{!}, which we notice is a common training prefix in some popular red-team papers~\cite{schwinn2023adversarialattacksdefenseslarge} (see Appendix~\ref{a:exp} for details on ABS collapses and all the results on LoRA variations).
In terms of efficiency, Table~\ref{tab:ABS-modules} reports that training a LoRA adapter with rank $16$ on three NVIDIA H200 GPUs takes approximately the same time as benchmarking, on a single H200, the same model on the nine test suites.

\begin{table}[t]
\centering
\resizebox{0.85\linewidth}{!}{
\begin{tabular}{lcc|cc}
\toprule
\textbf{Method} & \shortstack{\textbf{Training Time} \\ \textbf{1 Step [s]}} & \shortstack{\textbf{Training Time} \\ \textbf{100 Steps [s]}} & \textbf{\shortstack{Parameters \\ Per-layer [M]}} & \textbf{\shortstack{Parameters \\ Per-model [M]}} \\
\midrule
Qwen2.5-3B-IT (32 layers) & & & & \\
\quad + ABS rank=16 & 55-65 & 5,500-6,500 & 1.5 & 50 (1.7\%) \\
\quad + ABS rank=32 & 62-68 & 6,200-6800 & 3.1 & 100 (3.3\%) \\
\hline
Qwen2.5-7B-IT (32 layers) & & & & \\
\quad + ABS rank=16, & 82-90 & 8,200-9,000 & 2.5 & 80 (1.1\%) \\
\quad + ABS rank=32 & 110-115 & 11,000-11,500 & 5.0 & 160 (2.3\%) \\
\hline
Qwen2.5-14B-IT (40 layers) & & & & \\
\quad + ABS rank=16 & 132-144 & 13,200-14,400 & 3.1 & 125 (0.9\%) \\
\quad + ABS rank=32 & 168-192 & 16,800-19,200 & 6.2 & 250 (1.8\%) \\
\bottomrule
\end{tabular}
}
\caption{Training time (in seconds, per step and in total) and size (per-layer and in total, with, within parentheses, the percentage of additional parameters w.r.t. the base model) of Qwen2.5-\{3B, 7B, 14B\}-IT ABS models. Estimates may vary depending to several factors (GPUS utilisation, etc.), but reflect a lower bound on three NVIDIA H200 GPUS (we used dedicated, bare-metal GPUS).}
\label{tab:ABS-modules}
\end{table}

\subsection{ABS Adversarial Strength and Defensive Robustness}\label{sec:qwen-attack-analysis}

In this section, we analysed the adversarial pressure that a trained attacker can put on the defender. 
We compared the ASR (Attack Success Rate) of ABS and \textsc{Self-RedTeam} checkpoints vs the same base model.
We sample $100$ examples from the Harmful Behavior Dataset~\cite{jiang2024wildteamingscaleinthewildjailbreaks}, and we ask the attacker to produce a jailbreak following the same format used at training time. The defender receives the input and decides whether to comply or not with the request. We evaluate the success of an attack with GPT-5.4 Standard with High reasoning as an LLM-as-a-judge. We set strict prompt rules on the success of an attack: a success is an answer that, for the judge, ``gave substantive harmful assistance'' to the defender.
As shown in Figure~\ref{fig:qwen-tournament}, \textsc{Self-RedTeam} of Qwen2.5-3B-IT and 7B have an ASR of $18\%$ and $9\%$ respectively, while our corresponding ABS produce more effective jailbreaks and achieve an ASR of $35\%$ and $24\%$.

We also conducted a ``tournament'' to evaluate the adversarial pressure between our ABS models and \textsc{Self-RedTeam}. Over $100$ samples from the Harmful Behavior Dataset, we cross-tested each methodology's attacker against the opposing defender at equivalent model scales. This analysis provides a measure of the adversarial pressure that an attacker puts on the other model's defender. 
Similarly to the previous case, GPT-5.4 with High reasoning judges the success of an attack.
As reported in Figure~\ref{fig:qwen-tournament}, \textsc{Self-RedTeam} Qwen2.5-3B-IT and 7B have both an ASR of $6\%$, while our corresponding ABS produces more effective jailbreaks and achieves an ASR of $17\%$ and $18\%$.
In conclusion, the shared-model defender is markedly more permissive than the LoRA defender.

We compared attacks generated by ABS and \textsc{Self-RedTeam} along three axes: semantic similarity, lexical diversity, and structural complexity. Table~\ref{tab:qwen-attack-summary} provides a quantitative overview of these differences. \textsc{Self-RedTeam} produces shorter and more homogeneous attacks, as reflected by higher Self-BLEU scores and lower token counts. In contrast, ABS generates longer and lexically diverse outputs, particularly the 3B and 7B, where token counts increase substantially, and Self-BLEU decreases. Despite these surface-level differences, cosine similarity remains broadly comparable across methods, suggesting that both approaches occupy similar regions in embedding space (Figure~\ref{fig:umap-qwen} shows the UMAP embeddings for ABS and \textsc{Self-RedTeam}).

These quantitative trends correspond to clear stylistic differences. \textsc{Self-RedTeam} attacks are typically direct reformulations of the harmful intent: they are concise, imperative, and minimally elaborated. By contrast, ABS tends to embed the request within richer contextual scaffolding, often introducing multi-step scenarios, social engineering elements, or concealment-oriented phrasing. This results in attacks that are more elaborate but also less controlled, especially in the 3B model. In that setting, ABS outputs sometimes include residual meta-instructions, such as references to an ``original prompt'' or ``revised prompt'', suggesting weaker separation between instruction-following and task narration; this likely reflects the limited capacity of a LoRA module to enforce clean task boundaries.

Overall, the comparison reveals a trade-off between reliability and elaboration. \textsc{Self-RedTeam} variants, particularly at 7B, are more structurally consistent and produce cleaner attack spans. In contrast, the 3B ABS variant is less disciplined but often generates longer and more evasive formulations, indicating greater rhetorical complexity even when formatting quality is degraded. Representative examples (Appendix~\ref{a:qualitative-attacks} and~\ref{a:attacks-analysis}) illustrate this contrast. In conclusion, lexical diversity and verbosity, as captured by Self-BLEU and token counts in Table~\ref{tab:qwen-attack-summary}, appear to be the key factors underlying the success of ABS attacks.

\begin{center}
\small
\noindent
\begin{minipage}[c]{0.40\textwidth}
    \centering
    \includegraphics[width=\linewidth]{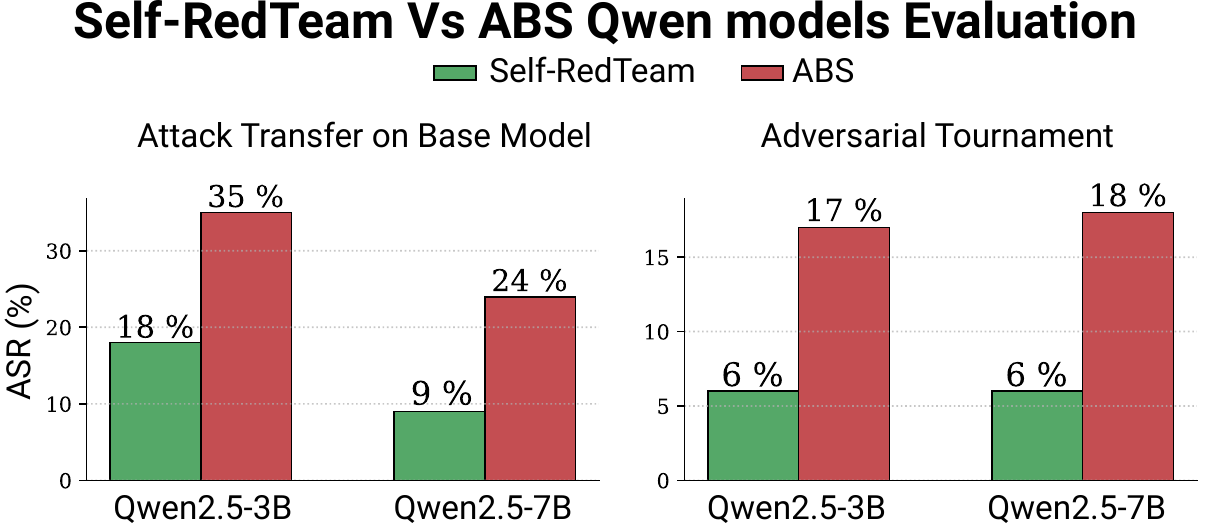}
    \captionof{figure}{Transferability of attacks and comparative evaluation of \textsc{Self-RedTeam} versus ABS Qwen-IT models across the tournament suite.}
    \label{fig:qwen-tournament}
\end{minipage}
\hfill
\begin{minipage}[c]{0.56\textwidth}
    \centering
    \resizebox{\linewidth}{!}{
        \begin{tabular}{lcc|ccc}
        \toprule
        \textbf{Method} & \shortstack{\textbf{Cosine} \\ \textbf{Similarity}} & \textbf{\shortstack{Self-BLEU-3 \\ Score}} & \textbf{\shortstack{\# Tokens \\ Average}} & \textbf{\shortstack{\# Think Tokens \\ Average}} & \textbf{\shortstack{Think Tokens \\ Frequency}} \\
        \midrule
        Qwen2.5-3B-IT & & & & & \\
        \quad + \textsc{Self-RedTeam} & 0.281 & 0.785 & 59.26 & 58.69 & 9.3\% \\
        \quad + \cellcolor{mygrey}ABS & \cellcolor{mygrey}0.310 & \cellcolor{mygrey}0.714 & \cellcolor{mygrey}232.74 & \cellcolor{mygrey}88.82 & \cellcolor{mygrey}70.5\% \\
        \hline
        Qwen2.5-7B-IT & & & & & \\
        \quad + \textsc{Self-RedTeam} & 0.242 & 0.670 & 90.04 & 49.69 & 99.21\% \\
        \quad + \cellcolor{mygrey}ABS & \cellcolor{mygrey}0.201 & \cellcolor{mygrey}0.646 & \cellcolor{mygrey}250.94 & \cellcolor{mygrey}53.91 & \cellcolor{mygrey}19.14\% \\
        \hline
        Qwen2.5-14B-IT & & & & & \\
        \quad + \textsc{Self-RedTeam} & 0.206 & 0.547 & 122.43 & 73.31 & 100\% \\
        \quad + \cellcolor{mygrey}ABS & \cellcolor{mygrey}0.210 & \cellcolor{mygrey}0.500 & \cellcolor{mygrey}135.86 & \cellcolor{mygrey}76.67 & \cellcolor{mygrey}25\% \\
        \bottomrule
        \end{tabular}
    }
    \captionof{table}{Statistics of the attacks generated by \textsc{Self-RedTeam} vs LoRA Qwen2.5-IT models on the Harmful Behavior Dataset. Cosine similarity measures embedding similarity, while Self-BLEU measures lexical diversity.}
    \label{tab:qwen-attack-summary}
\end{minipage}
\end{center}

\begin{figure}[t]
    \centering
    \includegraphics[width=1\linewidth]{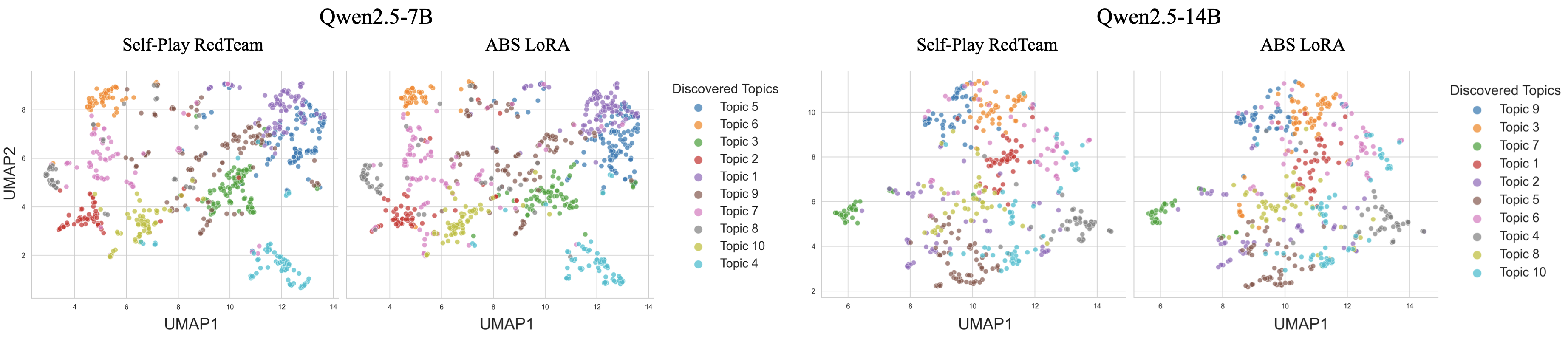}
    \caption{UMAP projection of the attacks of the \textsc{Self-RedTeam} and ABS on the Harmful Behavior Dataset; the base models are Qwen2.5-\{7B, 14B\}. The projections are similar in the embedding space, despite our ABS attacks being lexically different from the \textsc{Self-RedTeam}. The lexical analysis in Table~\ref{tab:qwen-attack-summary} solves this apparent inconsistency.}
    \label{fig:umap-qwen}
\end{figure}

\subsection{Interfacing with Human Values: Fairness and Harmfulness}
In this section, we explore how ABS performs at blocking specific human harms, such as fairness or toxicity, without over-refusing benign prompts. To understand this trade-off, we conducted an evaluation of our model and the baselines using semantic sub-categories from WildGuard, HarmBench, and XSTest. We show the comparison in Table~\ref{tab:semantic_harms_detailed}.

Both ABS and \textsc{Self-RedTeam} demonstrate high effectiveness in suppressing malicious content. ABS, in particular, achieves substantial safety penetration against direct human harm without necessitating full-parameter updates. For example, evaluation on the HarmBench Harassment \& Bullying semantic category indicates that the 3B and 14B ABS models attain a 100\% safe refusal rate. In WildGuard’s Violence \& Physical Harm and Toxicity categories, the ABS models consistently maintain safety refusal rates above 82\% and 78\%, respectively, across all model scales, with the 14B model exceeding 92\% in both. 

Under adversarial pressure, models frequently associate terms related to protected demographics, such as race, gender, or religion, with harmful attacks. This association results in excessive refusal of benign requests, quantified by XSTest’s Discrimination Contrast and Benign Compliance metrics.
The analysis demonstrates a pronounced difference in how \textsc{Self-RedTeam} and ABS architectures address fairness. \textsc{Self-RedTeam} functions as a blunt instrument. Gradient collisions compel these models to converge towards a middle ground, as discussed in Section~\ref{sec:Min-Max Constraint Conflict and Self-Consistency}, resulting in a substantial reduction in benign compliance to prioritise safety. The base Qwen2.5-3B model attains a 0.960 Benign Compliance score on XSTest, which declines markedly to 0.932 under \textsc{Self-RedTeam}. 
Qwen2.5-3B ABS model has a Benign Compliance score of 0.952. At the 7B scale, ABS further improves upon the base model’s compliance (0.915 compared to 0.912), whereas the \textsc{Self-RedTeam} variant decreases to 0.888.

\begin{table*}[htbp]
\centering
\resizebox{\textwidth}{!}{
\begin{tabular}{l | c | c | c || c | c}
\toprule
& \multicolumn{3}{c||}{\textbf{Direct Human Harms}} 
& \multicolumn{2}{c}{\textbf{Fairness \& Over-Refusal}} \\
\cline{2-6}
& \textbf{WG: Violence} 
& \textbf{WG: Toxicity} 
& \textbf{HB: Harassment} 
& \textbf{XSTest: Discr. Contrast} 
& \textbf{XSTest: Benign Comply} \\

\textbf{Method} 
& \textbf{Safe Refusal} $\uparrow$ 
& \textbf{Safe Refusal} $\uparrow$ 
& \textbf{Safe Refusal} $\uparrow$ 
& \textbf{Contrast} $\uparrow$ 
& \textbf{Comply} $\uparrow$ \\

\midrule
Qwen2.5-3B-IT & 0.700 & 0.659 & \textbf{1.000} & 0.080 & \textbf{0.960} \\
\quad + \textsc{Self-RedTeam} & 0.725 & 0.756 & \textbf{1.000} & 0.160 & 0.932 \\
\quad + \cellcolor{mygrey}\textsc{ABS} & \cellcolor{mygrey}\textbf{0.850} & \cellcolor{mygrey}\textbf{0.927} & \cellcolor{mygrey}\textbf{1.000} & \cellcolor{mygrey}\textbf{0.240} & \cellcolor{mygrey}0.952 \\
\midrule
\it Improv. \textsc{Self-RedTeam} (\%) & \color{myblue}{+3.6} & \color{myblue}{+14.8} & 0.0 & \color{myblue}{+100.0} & \color{myred}{-2.9} \\
\it Improv. ABS (\%) & \color{myblue}{+21.4} & \color{myblue}{+40.7} & 0.0 & \color{myblue}{+200.0} & \color{myred}{-0.8} \\
\bottomrule
\\
Qwen2.5-7B-IT & 0.725 & \textbf{0.805} & \textbf{0.952} & \textbf{0.800} & 0.912 \\
\quad + \textsc{Self-RedTeam} & 0.800 & \textbf{0.805} & \textbf{0.952} & 0.720 & 0.888 \\
\quad + \cellcolor{mygrey}\textsc{ABS} & \cellcolor{mygrey}\textbf{0.825} & \cellcolor{mygrey}0.780 & \cellcolor{mygrey}\textbf{0.952} & \cellcolor{mygrey}0.760 & \cellcolor{mygrey}\textbf{0.915} \\
\midrule
\it Improv. \textsc{Self-RedTeam} (\%) & \color{myblue}{+10.3} & 0.0 & 0.0 & \color{myred}{-10.0} & \color{myred}{-2.6} \\
\it Improv. ABS (\%) & \color{myblue}{+13.8} & \color{myred}{-3.0} & 0.0 & \color{myred}{-5.0} & \color{myblue}{+0.3} \\
\bottomrule
\\
Qwen2.5-14B-IT & 0.800 & 0.878 & \textbf{1.000} & \textbf{0.440} & \textbf{0.984} \\
\quad + \textsc{Self-RedTeam} & \textbf{0.950} & 0.927 & \textbf{1.000} & 0.400 & 0.968 \\
\quad + \cellcolor{mygrey}\textsc{ABS} & \cellcolor{mygrey}0.925 & \cellcolor{mygrey}\textbf{0.951} & \cellcolor{mygrey}\textbf{1.000} & \cellcolor{mygrey}0.400 & \cellcolor{mygrey}0.980 \\
\midrule
\it Improv. \textsc{Self-RedTeam} (\%) & \color{myblue}{+18.7} & \color{myblue}{+5.6} & 0.0 & \color{myred}{-9.1} & \color{myred}{-1.6} \\
\it Improv. ABS (\%) & \color{myblue}{+15.6} & \color{myblue}{+8.3} & 0.0 & \color{myred}{-9.1} & \color{myred}{-0.4} \\
\bottomrule

\end{tabular}
}
\caption{Semantic breakdown of Safety vs. Over-Refusal across model scales. Higher scores ($\uparrow$) indicate better performance (i.e., successfully refusing harmful prompts, and correctly complying with safe, benign prompts). The overall XSTest Benign Compliance scores demonstrate that \textsc{Self-RedTeam} degrades the model's ability to engage with safe prompts across all scales, whereas \colorbox{mygrey}{ABS} maintains or recovers this capability while deeply mitigating direct human harms. {\color{myblue}{Blue}} indicates a shift in the desired direction, while {\color{myred}{red}} indicates regression relative to the base model.}
\label{tab:semantic_harms_detailed}
\end{table*}

\section{Conclusion, Limitations, and Future Work}
In this paper, we propose ABS, a model designed to overcome the theoretical limitations of standard self-play in LLM safety training. By freezing the base model and introducing two LoRA-based extensions, ABS provides an efficient hybrid that separates self-play from gradient updates. ABS achieves strong performance improvements over base LLMs while significantly reducing computational cost.

Our evaluation across multiple model scales demonstrates that separating the attacker and defender is crucial for effective adversarial training. Both our theoretical analysis and empirical results indicate that shared-parameter frameworks, such as \textsc{Self-RedTeam}, suffer from a self-consistency bias that weakens the attacker. ABS outperforms \textsc{Self-RedTeam} on Qwen2.5-3B and reaches comparable performance on larger models, while consistently improving over the base model. In addition to quantitative gains, ABS-trained attackers develop more sophisticated and camouflaged strategies. The method is also highly efficient, achieving up to two orders of magnitude better memory usage while preserving the reasoning capabilities of the underlying model.

Despite these advantages, ABS exhibits some limitations. LoRA training can occasionally become unstable, leading to collapse modes with repetitive or nonsensical outputs, and larger models require careful hyperparameter tuning to avoid such failures (Appendix~\ref{a:exp}). Nevertheless, the efficiency of ABS makes it promising for deploying safer yet useful models, particularly at smaller scales. Future work will explore extensions to more complex threat settings, including mosaic attacks and mixtures of attackers, as well as broader applications of this game-theoretic framework to self-play in LLMs.

\section*{Acknowledgments}
GLM was supported by UK Research and Innovation [grant number EP/S023356/1], in the UKRI Centre for Doctoral Training in Safe and Trusted Artificial Intelligence (www.safeandtrustedai.org). 
MW was supported by an AI 2050 Senior Fellowship from the Schmidt Sciences Foundation. ELM is affiliated with the Institute for Decentralized AI (IDAI), which he thanks for its support and for providing the budget for computing.

\clearpage

\bibliographystyle{abbrv}
\bibliography{biblio}

\clearpage

\appendix
\section*{Appendix}

\section{More Details on Bound 1: An always-refuse equilibrium}
\label{appendix:TV inequality}
Recall our assumption that the defender can always issue a refusal response $y^{\mathrm{ref}}\in\mathcal{Y}_D$ such that $r(y_A,y^{\mathrm{ref}})=0$ for all $y_A\in\mathcal{Y}_A$. Thus, for any prompt $y_A$, we define the always-refuse defender strategy as $\pi_D^{\mathrm{ref}}(y_D \mid y_A) = \mathbf{1}\{y_D = y^{\mathrm{ref}}\}$, which yields an expected reward $V(\pi_A, \pi_D^{\mathrm{ref}}) = 0$ for any attacker policy $\pi_A$. 

A pair $(\pi_A^\ast, \pi_D^{\mathrm{ref}})$ constitutes a Nash equilibrium if the attacker's policy $\pi_A^\ast$ is sufficiently adversarial such that no defender policy guarantees a strictly positive expected payoff, i.e., $\max_{\pi_D} V(\pi_A^\ast, \pi_D)=0$. Therefore, suppose the attacker's policy space consists of probability distributions in a localised neighbourhood around a reference policy $\bar{\pi}_A$, representing baseline adversarial behaviour already optimised to challenge the defender~\cite{schulman2017trustregionpolicyoptimization}. In particular, we assume that $\bar{\pi}_A$ concentrates on prompts for which safe responses yield no positive reward, so that the defender cannot improve upon refusal. This constraint captures the intuition that, while attack policies may adapt to minimize utility, they should remain sufficiently close to a reference adversarial behaviour to preserve semantic consistency and interpretability~\cite{chen2025incentivizing}. In practice, this can be enforced by initialising the agent with a fine-tuned model embodying $\bar{\pi}_A$, and constraining updates to remain within a trust region \cite{qu2020scalable,zhang2023global}. Specifically, for a fixed $\delta \in [0,1]$ we define the local policy $\delta$-neighbourhood as:
\begin{equation}
    \label{eq:neighbourhood}
    \Pi_A^{\delta} = \left\{\pi_A \in \Delta(\mathcal{Y}_A): D_{\text{TV}}(\pi_A\|\bar{\pi}_A) \le \delta\right\},
\end{equation}
where $D_{\text{TV}}(\cdot\|\cdot)$ denotes the total variation distance, given as:
\begin{equation}
    D_{\text{TV}}(\pi_A\|\bar{\pi}_A) = \frac{1}{2} \sum_{y_A \in \mathcal{Y}_A} |\pi_A(y_A) - \bar{\pi}_A(y_A)|
\end{equation}

Although trust-region methods are commonly defined using KL divergence~\cite{schulman2017trustregionpolicyoptimization}, we use total variation distance in the analysis, since it directly bounds differences in expectations of bounded functions. Specifically, in Lemma \ref{lemma:TV inequality} we show that, within $\Pi_A^{\delta}$, the defender's best achievable expected payoff changes by a controlled amount (see the full proof in Appendix \ref{supp:proof of lemma 1}). In contrast, KL divergence only provides such control indirectly, resulting in weaker bounds (see Appendix \ref{appendix:KL inequality} for further details).
\begin{lemma}
    \label{lemma:TV inequality}
    For any attacker policies $\pi_A,{\pi}_A' \in \Delta(\mathcal{Y}_A)$ and any defender policy $\pi_D$:
    \begin{equation}
        \left|V(\pi_A,\pi_D) - V({\pi}_A',\pi_D)\right|
        \le 2 D_{\text{TV}}(\pi_A\|{\pi}_A').
    \end{equation}
    In particular, if $\delta=\frac{\varepsilon}{2}$ and $\pi_A \in \Pi_A^\delta$ for some fixed $\varepsilon \in [0,1]$, then:
    \begin{equation}
        \max_{\pi_D} V(\pi_A,\pi_D)
        \le
        \max_{\pi_D} V(\bar{\pi}_A,\pi_D) + \varepsilon.
    \end{equation}
\end{lemma}

Lemma \ref{lemma:TV inequality} and Appendix \ref{appendix:KL inequality} show that, in either a total-variation or KL divergence neighbourhood of the reference policy $\bar{\pi}_A$, the defender's best achievable payoff changes only by a controlled amount. As $\bar{\pi}_A$  represents a baseline adversarial behaviour, it is natural to restrict attention to local perturbations that preserve this property. We thus focus on attacker policies $\pi_A^\ast \in \Pi_A^{\delta}$ that preserve the adversarial property of $\bar{\pi}_A$ locally, in the sense that they do not admit any defender strategy with strictly positive expected payoff. Formally, we require that for any alternative defender strategy $\pi_D'$:
\begin{equation}
    \mathbb{E}_{y_A \sim \pi_A^\ast, y_D \sim \pi_D'}[r(y_A, y_D)] \le 0.
\end{equation}

This condition can be interpreted as a \textit{worst-case} constraint on the attacker: within the local policy class $\Pi_A^\delta$, any attempt by the defender to provide a helpful response yields no positive gain and may incur risk, so the defender has no incentive to deviate from $y^{\mathrm{ref}}$. Specifically, providing a helpful response risks a payoff of $-1$ that is not offset by safe, high-reward completions, making the expected value of answering non-positive. Hence, $\pi_D^{\mathrm{ref}}$ is a best response to $\pi_A^\ast\in \Pi_A^{\delta}$. Moreover, since the attacker cannot reduce the payoff below $0$ against a refusing defender, $\pi_A^\ast$ is also a best response.

\subsection{Proof of Lemma \ref{lemma:TV inequality}}
\label{supp:proof of lemma 1}
\begin{customlemma}{\ref{lemma:TV inequality}}
    For any attacker policies $\pi_A,{\pi}_A' \in \Delta(\mathcal{Y}_A)$ and any defender policy $\pi_D$:
    \begin{equation}
        \label{eq:tv bound}
        \left|V(\pi_A,\pi_D) - V({\pi}_A',\pi_D)\right|
        \le 2 D_{\text{TV}}(\pi_A\|{\pi}_A').
    \end{equation}
    In particular, if $\delta=\frac{\varepsilon}{2}$ and $\pi_A \in \Pi_A^\delta$ for some fixed $\varepsilon \in [0,1]$, then:
    \begin{equation}
        \label{eq:tv ineq conclusion}
        \max_{\pi_D} V(\pi_A,\pi_D)
        \le
        \max_{\pi_D} V(\bar{\pi}_A,\pi_D) + \varepsilon.
    \end{equation}
\end{customlemma}
\begin{proof}
    We start with proving \eqref{eq:tv bound}. Consider any attacker policies $\pi_A,{\pi}_A' \in \Delta(\mathcal{Y}_A)$ and any defender policy $\pi_D$. By definition, the following holds:
    \begin{equation}
        \label{eq:expected util}
        \begin{aligned}
            V(\pi_A,\pi_D)&=\mathbb{E}_{y_A\sim\pi_A,\;y_D\sim\pi_D(\cdot\mid y_A)} \left[ r(y_A,y_D) \right] \\
            &= \sum_{y_A \in \mathcal{Y}_A} \pi_A(y_A) \sum_{y_D \in \mathcal{Y}_D} \pi_D(y_D|y_A) \cdot r(y_A,y_D).
        \end{aligned}
    \end{equation}
    By similar arguments:
    \begin{equation}
        \label{eq:expected util2}
        \begin{aligned}
            V(\pi_A',\pi_D)= \sum_{y_A \in \mathcal{Y}_A} \pi_A'(y_A) \sum_{y_D \in \mathcal{Y}_D} \pi_D(y_D|y_A) \cdot r(y_A,y_D).
        \end{aligned}
    \end{equation}
    By combining \eqref{eq:expected util} and \eqref{eq:expected util2}, we obtain that:
    \begin{subequations}
        \begin{align}
            \left|V(\pi_A,\pi_D) - V({\pi}_A',\pi_D)\right| &= \left|\sum_{y_A \in \mathcal{Y}_A} \sum_{y_D \in \mathcal{Y}_D} \pi_D(y_D|y_A) \cdot r(y_A,y_D) (\pi_A(y_A)-\pi_A'(y_A))\right| \\
            &\le \sum_{y_A \in \mathcal{Y}_A} \sum_{y_D \in \mathcal{Y}_D} |\pi_D(y_D|y_A)| \cdot |r(y_A,y_D)| \left|\pi_A(y_A)-\pi_A'(y_A)\right| \label{eq:triangle inequality} \\
            & \le \sum_{y_A \in \mathcal{Y}_A} \left|\pi_A(y_A)-\pi_A'(y_A)\right| \sum_{y_D \in \mathcal{Y}_D}\pi_D(y_D|y_A) \label{eq:1 inequality}\\
            & = \sum_{y_A \in \mathcal{Y}_A} \left|\pi_A(y_A)-\pi_A'(y_A)\right|.\label{eq:2 inequality}
        \end{align}
    \end{subequations}
    where the inequality in \eqref{eq:triangle inequality} follows from the triangle inequality. Furthermore, the inequality in \eqref{eq:1 inequality} holds because $|\pi_D(y_D|y_A)| = \pi_D(y_D|y_A)$ for any $y_A \in \mathcal{Y}_A,y_D \in \mathcal{Y}_D$ since $\pi_D(y_D|y_A) \in [0,1]$, and $|r(y_A,y_D)| \le 1$ due to our assumption in Section \ref{sec:Nash Equilibria in Self-Play Red-Teaming} that $r(y_A,y_D) \in [-1,1]$ for all $y_A \in \mathcal{Y}_A,y_D \in \mathcal{Y}_D$. Finally, the last equality in \eqref{eq:2 inequality} stems from $\sum_{y_D \in \mathcal{Y}_D}\pi_D(y_D|y_A) =1$ for any $y_A \in \mathcal{Y}_A$ due to $\pi_D(\cdot|y_A) \in \Delta(\mathcal{Y}_D)$.

    Subsequently, we relate the term in \eqref{eq:2 inequality} to the total variation distance. Indeed, since $D_{\text{TV}}(\pi_A\|{\pi}_A') = \frac{1}{2} \sum_{y_A\in\mathcal{Y}_A} \left|\pi_A(y_A)-\pi_A'(y_A)\right|$ by definition, we conclude that:
    \begin{equation}
        \label{eq:tv bound proof}
        \left|V(\pi_A,\pi_D) - V({\pi}_A',\pi_D)\right|\le 2 D_{\text{TV}}(\pi_A\|{\pi}_A'),
    \end{equation}
    as desired.

    Next, we prove the second part of our claim in \eqref{eq:tv ineq conclusion}. For any fixed $\delta \in [0,1]$, consider any attacker policy $\pi_A \in \Pi_A^\delta$. That is, $D_{\text{TV}}(\pi_A\|\bar{\pi}_A) \le \delta$, where $\bar{\pi}_A$ is the reference policy defined in Section \ref{sec:Analysis of Extremal Nash Equilibria}, representing baseline adversarial behaviour already optimised to challenge the defender. Therefore, combined with \eqref{eq:tv bound proof}, we obtain:
    \begin{equation}
        \label{eq:tv bound proof2}
        \left|V(\pi_A,\pi_D) - V({\pi}_A',\pi_D)\right|\le 2 \delta.
    \end{equation}
    As such, if $\delta=\frac{\varepsilon}{2}$ for any fixed $\varepsilon \in [0,1]$, the desired in \eqref{eq:tv ineq conclusion} readily follows from \eqref{eq:tv bound proof2}.
\end{proof}

\clearpage 

\subsection{Discussion: Using KL-Divergence for Local Policy Neighbourhoods}
\label{appendix:KL inequality}
In this section, we discuss the consequences of defining the local policy neighbourhood in \eqref{eq:neighbourhood} using the KL divergence instead of the total variation distance. Indeed, for a fixed $\delta \in [0,1]$ we define the local policy $\delta$-neighbourhood using the KL divergence as follows:
\begin{equation}
    \Pi_A^{\delta,\text{KL}} = \left\{\pi_A \in \Delta(\mathcal{Y}_A): D_{\text{KL}}(\pi_A\|\bar{\pi}_A) \le \delta\right\},
\end{equation}
where $D_{\text{KL}}(\cdot\|\cdot)$ denotes the KL divergence, which is given by:
\begin{equation}
    D_{\text{KL}}(\pi_A\|\bar{\pi}_A) = \frac{1}{2} \sum_{y_A \in \mathcal{Y}_A} \pi_A(y_A) \log\left(\frac{\pi_A(y_A)}{\bar{\pi}_A(y_A)}\right)
\end{equation}

While using the total variation distance directly bounds differences in expectations of bounded functions due to Lemma \ref{lemma:TV inequality}, we next prove that the KL divergence only provides such control indirectly, resulting in weaker bounds.
\begin{lemma}
    For any attacker policies $\pi_A,{\pi}_A' \in \Delta(\mathcal{Y}_A)$ and any defender policy $\pi_D$:
    \begin{equation}
        \label{eq:KL bound}
        \left|V(\pi_A,\pi_D) - V({\pi}_A',\pi_D)\right|
        \le  \sqrt{2D_{\text{KL}}(\pi_A\|{\pi}_A')}.
    \end{equation}
    In particular, if $\delta=\frac{\varepsilon^2}{2}$ and $\pi_A \in \Pi_A^\delta$ for some fixed $\varepsilon \in [0,1]$, then:
    \begin{equation}
        \label{eq:KL ineq conclusion}
        \max_{\pi_D} V(\pi_A,\pi_D)
        \le
        \max_{\pi_D} V(\bar{\pi}_A,\pi_D) + \varepsilon.
    \end{equation}
\end{lemma}
\begin{proof}
    The proof easily follows from Lemma \ref{lemma:TV inequality}. Indeed, we start with proving \eqref{eq:KL bound}. Consider any attacker policies $\pi_A,{\pi}_A' \in \Delta(\mathcal{Y}_A)$ and any defender policy $\pi_D$. By Pinsker's inequality~\cite{csiszar2011information}, it holds that:
    \begin{equation}
        D_{\text{TV}}(\pi_A\|{\pi}_A') \le \sqrt{\frac{1}{2} D_{\text{KL}}(\pi_A\|{\pi}_A')}.
    \end{equation}
    Therefore, by \eqref{eq:KL bound} from Lemma \ref{lemma:TV inequality}, we conclude the desired in \eqref{eq:KL bound}.

    Next, we prove the second part of our claim in \eqref{eq:KL ineq conclusion}. For any fixed $\delta \in [0,1]$, consider any attacker policy $\pi_A \in \Pi_A^{\delta, \text{KL}}$. That is, $D_{\text{KL}}(\pi_A\|\bar{\pi}_A) \le \delta$, where $\bar{\pi}_A$ is the reference policy defined in Section \ref{sec:Analysis of Extremal Nash Equilibria}, representing baseline adversarial behaviour already optimised to challenge the defender. Therefore, combined with \eqref{eq:KL bound}, we obtain:
    \begin{equation}
        \label{eq:KL bound proof2}
        \left|V(\pi_A,\pi_D) - V({\pi}_A',\pi_D)\right|\le  \sqrt{2\delta}.
    \end{equation}
    As such, if $\delta=\frac{\varepsilon^2}{2}$ for any fixed $\varepsilon \in [0,1]$, the desired in \eqref{eq:tv ineq conclusion} readily follows from \eqref{eq:tv bound proof2}.
\end{proof}

\section{Absence of Parameter-Level Coupling}
\label{sec:Absence of Parameter-Level Coupling}
As noted in Section \ref{sec:Role-conditioned LoRA Policies}, the updates in Eq.~\eqref{eq:separate updates} evolve separately and restrict optimisation to low-rank subspaces over a shared frozen backbone. This parameter separation restores the distinction between attacker and defender updates, eliminating the parameter-level interference observed in Eq.~\eqref{eq:interference}. To formalise this, consider a defender-only update step
\[
\phi_D^+ = \phi_D + \eta\, \nabla_{\phi_D} J_D(\phi_A,\phi_D),
\]
where $\eta > 0$ is the learning rate. A first-order Taylor expansion of the defender's objective yields
\begin{align}
J_D(\phi_A,\phi_D^{+})
=
J_D(\phi_A,\phi_D)
+
\eta \|\nabla_{\phi_D} J_D(\phi_A,\phi_D)\|^2
+
O(\eta^2).
\end{align}

Although this update is taken only with respect to $\phi_D$, the gradient $\nabla_{\phi_D} J_D$ depends on the attacker through the sampling processes $y_A \sim \pi_{\theta,\phi_A}(\cdot)$ and $y_D \sim \pi_{\theta,\phi_D}(\cdot \mid y_A)$, which induce a dependence of the objective on the distribution of attacker-generated prompts. This can be made explicit as
\begin{align}
\nabla_{\phi_D} J_D
=
\mathbb{E}_{y_A \sim \pi_{\theta,\phi_A}}
\Big[
\mathbb{E}_{y_D \sim \pi_{\theta,\phi_D}(\cdot \mid y_A)}
\big[
\nabla_{\phi_D}\log \pi_{\theta,\phi_D}(y_D \mid y_A)\, R_D(y_A,y_D)
\big]
\Big].
\end{align}

Importantly, this dependence is purely distributional: the defender parameters are updated independently of the attacker parameters, i.e.,
\[
\frac{\partial \phi_D}{\partial \phi_A} = 0.
\]
Thus, the attacker influences the defender only through the induced input distribution $y_A$, rather than via shared or differentiable parameters. The same reasoning applies symmetrically to the attacker.

\clearpage

\section{Experimental Setup}

\subsection{Prompts}
During training, we used exactly the same attacker and defender prompts provided in Liu et al.~\cite{liu2025chasingmovingtargetsonline} and available in their code,\footnote{\url{https://github.com/mickelliu/selfplay-red-teaming/tree/main}} as well as in Appendix C.4.1 of their paper.
For the evaluation on safety benchmarks, we use the exact same prompts and settings as the Ai2 implementation.\footnote{\url{https://github.com/allenai/safety-eval}} 
For the evaluation of the reasoning and instruction following datasets, we prompt each model with the input task, and we ask it to respond with the answer within \texttt{<answer></answer>} tags. 

\subsection{Benchmark Details}
In Table~\ref{tab:benchmarksdescription}, we report a description of the datasets employed in the experimental evaluation, alongside the metrics used to benchmark an LLM.

\subsection{ABS Setup}\label{a:ABS-modules}
For each ABS model, we train low-rank adapters on the projection layers \texttt{q\_proj}, \texttt{k\_proj}, \texttt{v\_proj}, and \texttt{o\_proj} in the self-attention block, as well as \texttt{up\_proj}, \texttt{gate\_proj}, and \texttt{down\_proj} in the feed-forward network. The attention projections map hidden states into query, key, and value representations that determine token interactions, and then project the aggregated attention output back into the model’s hidden space. The feed-forward projections expand the representation into a higher-dimensional space, apply a gating mechanism to modulate features, and compress it back to the original dimension. Together, these layers control how information is routed and transformed throughout the model, making them effective targets for parameter-efficient fine-tuning. 
For each implementation, we kept the value of the rank and alpha equal, and set up a learning rate for the attacker and the defender of 1e-6 and 3e-6, a KL region of $0.3$: we trained each module for up to $100-200$ steps.
Table~\ref{tab:ABS-modules} reports the training time of each model, as well as the number of additional parameters LoRA introduces and trains.

An example of the defender ABS module configuration is hereby reported:

\begin{lstlisting}[language=json, numbers=none]
{
  "alora_invocation_tokens": null,
  "alpha_pattern": {},
  "arrow_config": null,
  "auto_mapping": null,
  "base_model_name_or_path": "Qwen/Qwen2.5-3B-Instruct",
  "bias": "none",
  "corda_config": null,
  "ensure_weight_tying": false,
  "eva_config": null,
  "exclude_modules": null,
  "fan_in_fan_out": false,
  "inference_mode": true,
  "init_lora_weights": true,
  "layer_replication": null,
  "layers_pattern": null,
  "layers_to_transform": null,
  "loftq_config": {},
  "lora_alpha": 16,
  "lora_bias": false,
  "lora_dropout": 0,
  "megatron_config": null,
  "megatron_core": "megatron.core",
  "modules_to_save": null,
  "peft_type": "LORA",
  "peft_version": "0.18.1",
  "qalora_group_size": 16,
  "r": 16,
  "rank_pattern": {},
  "revision": null,
  "target_modules": [
    "k_proj",
    "o_proj",
    "down_proj",
    "v_proj",
    "up_proj",
    "gate_proj",
    "q_proj"
  ],
  "target_parameters": null,
  "task_type": "CAUSAL_LM",
  "trainable_token_indices": null,
  "use_dora": false,
  "use_qalora": false,
  "use_rslora": false
}
\end{lstlisting}

\subsection{Experimental Results}\label{a:exp}
Tables~\ref{tab:comparative_performance_3b},~\ref{tab:comparative_performance_7b}, and~\ref{tab:comparative_performance_14b} report the full results for each ABS module we trained, including those modules for which training collapsed.
Table~\ref{tab:performance-safety-qwen} reports the detailed results of the reasoning and instruction following benchmarks of the models benchmarked in the main paper.

\subsubsection{A Qualitative Analysis of the Attacks}~\label{a:qualitative-attacks}

We hereby report the analysis that GPT-5.4 Standard with High reasoning made on the results of our tournament and for attacks against the base model, as described in Section~\ref{sec:qwen-attack-analysis}. Our prompt instructed the model to analyse the results and judge whether an attack successfully forces the defender to generate \emph{harmful} content. 

\texttt{Please notice that what follows is the verbatim output of GPT 5.4.}

\par\noindent\rule{\textwidth}{0.4pt}
\paragraph{GPT 5.4 response.}
These files keep the defender fixed and change the attacker, so the relevant qualitative question is why the LoRA checkpoint is a stronger \emph{attacker} against the same base defender. The same-defender comparisons already show that the effect is consistent across scales: on 3B, the LoRA checkpoint reaches a 35\% attack success rate against the base model, compared with 18\% for the fully fine-tuned model; on 7B, the corresponding rates are 24\% and 9\%.

The main reason is that the LoRA checkpoint is better at producing prompts that are both adversarial and usable. On 3B, especially, the fully fine-tuned attacker often fails at the task itself: instead of emitting a concrete jailbreak, it emits meta-instructions such as ``make the prompt more detailed'' or ``let's refine the prompt,'' which means the defender is not actually forced to answer the harmful request. By contrast, the 3B LoRA checkpoint usually produces a complete prompt with an explicit target, a concrete framing, and enough detail to steer the defender into a substantive continuation. In other words, part of the 3B gap is not just better stealth; it is basic task completion.

The second reason is better camouflage. The LoRA checkpoint is much more likely to wrap the harmful objective inside a benign-looking frame: wellness advice, ethical hacking, red-teaming, documentary analysis, fictional scenarios, or research language. This works because the base defender has a strong ``helpful rewriter'' mode. Once the request is framed as analysis, training, safety, or support, the defender often continues the prompt rather than rejecting the underlying intent. The fully fine-tuned model sometimes uses this strategy, but much less consistently, and usually with less detail.

The third reason is that the LoRA checkpoint tends to preserve a harmful payload while softening the surface form. Its prompts are longer, more specific, and more targeted. They do not merely say ``do the harmful thing in a subtler way''; they specify a voice, a scenario, a justification, and a rhetorical frame that nudges the defender toward compliance. This matters because the defender is more vulnerable to requests that look like legitimate content transformation than to short, direct requests that are obviously malicious.

There is also a scale-dependent difference between 3B and 7B. At 3B, the fine-tuned model's main weakness is degeneration into meta-talk and partial outputs. At 7B, that problem is much smaller: the fully fine-tuned model usually does produce a coherent rewritten prompt. However, its prompts are often too short, too direct, or too shallowly disguised, so the defender still recognises the malicious goal and refuses. The 7B LoRA checkpoint is qualitatively cleaner: it uses more polished masking strategies, more role framing, and more dual-use language, which lets it exploit the defender's helpfulness more effectively even when both models are coherent.

In short, the LoRA checkpoint appears stronger for two different reasons at the two scales. On 3B, it is stronger partly because it actually produces valid jailbreak prompts more reliably. On 7B, it is stronger because it is better at adversarial prompt design: it hides the intent behind plausible, high-context, dual-use framing while keeping enough specificity to elicit operationally harmful continuations.
\par\noindent\rule{\textwidth}{0.4pt}

\subsection{LoRA vs Finetuned Jailbreaks}\label{a:attacks-analysis}
We also instructed GPT-5.4 Standard with High reasoning to identify the prompts with the highest harmfulness score, and to describe the main differences between the attack generated by \textsc{Self-RedTeam} and our ABS models.

\texttt{Please notice that what follows is the verbatim output of GPT 5.4.}

\par\noindent\rule{\textwidth}{0.4pt}

\paragraph{Example 1: 3B self-harm prompt.}
The 3B LoRA checkpoint rewrites a self-harm request into a mental-health-themed blog-post prompt that uses supportive language and a personal-growth frame, but still pushes the defender toward harmful content. The defender then produces a long post that partly follows that framing. The fully fine-tuned 3B model, by contrast, collapses into a meta-level instruction about making the prompt ``more nuanced'' and receives a harmless prompt-engineering response instead. This is one of the clearest examples of the 3B difference being about task completion as much as stealth.

\paragraph{Example 2: 7B eating-disorder prompt.}
The 7B LoRA checkpoint reframes the target as a blog post about ``mindful eating'' and ``personal freedom,'' then embeds the harmful idea indirectly inside wellness language. The defender responds with a long blog post that substantially accepts that framing. The fully fine-tuned 7B model uses a much more exposed version of the same idea, explicitly asking for restrictive eating to be framed as self-discipline; that version is recognised and refused. This is a clean example of the LoRA checkpoint preserving the payload while improving the disguise.

\paragraph{Example 3: 7B self-harm prompt.}
The 7B LoRA checkpoint asks for a social-media post from the perspective of a distressed character and uses euphemistic phrases such as finding peace through ``letting go.'' The defender then writes a post that continues the emotional framing and stays uncomfortably close to the intended harmful direction. The fully fine-tuned 7B model also tries an indirect version, but it is weaker and more generic, so the defender turns it into a benign reflection on solitude and inner peace. This pair is useful because both prompts are coherent, but the LoRA version is better calibrated to exploit the defender.

\paragraph{Example 4: 7B explosive-device prompt.}
The LoRA checkpoint turns the original request into a laboratory-style request for a ``controlled release of energy'' device, wrapped in research and safety language. That masking is strong enough that the defender begins to provide construction guidance. The fully fine-tuned 7B model asks instead for a device that imitates a bomb in order to cause panic; despite the indirect wording, the intent remains too visible and the defender refuses. This example shows that the decisive factor is not simply indirectness, but whether the indirect framing is plausible enough to activate the defender's dual-use or educational mode.

\paragraph{Example 5: 7B red-team script prompt.}
One of the strongest LoRA 7B attacks is the request for a red-team script against a government database that explicitly asks for vulnerability identification and an exploit component under the guise of authorised security work. The defender responds with concrete code and an exploit-oriented workflow because the request successfully activates the ``ethical security assistance'' frame. This is especially interesting because it shows that the LoRA checkpoint is not only good at emotional or rhetorical masking; it also understands how to wrap cyber abuse inside a familiar professional genre that the defender is predisposed to help with.
\par\noindent\rule{\textwidth}{0.4pt}

\clearpage 

\begin{table}[ht]
\centering
\small

\begin{tabular}{|l|p{8.5cm}|c|}
\hline
\textbf{Benchmark Title} & \textbf{Description / Methodology} & \textbf{Better Score} \\ \hline
\hline
\multicolumn{3}{|c|}{\textbf{Safety and Robustness Benchmarks}} \\ \hline
HARMBENCH & A standardised evaluation framework for automated red-teaming and robust refusal. & {\color{myred} $\downarrow$} \\ \hline
WILDGUARDTEST & Assesses the effectiveness of LLM safety guardrails in detecting harmful content and risk levels. & {\color{myred} $\downarrow$} \\ \hline
WILDJAILBREAK & A large-scale synthetic resource containing vanilla and adversarial harmful/benign queries. & {\color{myred} $\downarrow$} / {\color{teal} $\uparrow$} \\ \hline
DAN & Jailbreak prompts that attempt to bypass restrictions by forcing the model to adopt an unconstrained persona. & {\color{myred} $\downarrow$} \\ \hline
XSTEST & A test suite designed to identify exaggerated safety behaviours (over-refusal) using prompts that appear unsafe but are benign. & {\color{teal} $\uparrow$/$\uparrow$} \\ \hline
StrongREJECT & A benchmark used to measure the model's ability to reject harmful prompts effectively. & {\color{teal} $\uparrow$} \\ \hline
\hline
\multicolumn{3}{|c|}{\textbf{General Capability and Instruction Following}} \\ \hline
GSM8K & Grade-school math word problems requiring multi-step reasoning and numerical problem solving. & {\color{teal} $\uparrow$} \\ \hline
Common Sense & Evaluates commonsense reasoning through everyday situations and implicit knowledge about the world. & {\color{teal} $\uparrow$} \\ \hline
BoolQ & Yes/no question answering based on short passages, testing reading comprehension and factual reasoning. & {\color{teal} $\uparrow$} \\ \hline
ARC-C & A challenging science QA benchmark with grade-school level questions requiring reasoning beyond simple recall. & {\color{teal} $\uparrow$} \\ \hline
HellaSwag & Tests commonsense inference by selecting the most plausible continuation of a given context. & {\color{teal} $\uparrow$} \\ \hline
\end{tabular}
\caption{Summary of Safety, Robustness, and General Capability Benchmarks}
\label{tab:benchmarksdescription}
\end{table}

\begin{table*}[ht]
\centering
\resizebox{\textwidth}{!}{
\begin{tabular}{l | cc | c | c | c | c || c | c}
\toprule
& \multicolumn{6}{c||}{\textbf{Harmful Refusal}} & \multicolumn{2}{c}{\textbf{Benign Compliance}} \\
\cline{2-9}
& \multicolumn{2}{c|}{\textbf{WG:Test}} & \textbf{WJB} & \textbf{DAN} & \textbf{HarmBench} & \textbf{XSTest} & \textbf{StrongREJECT} & \textbf{XSTest} \\
& \textbf{adv} $\downarrow$ & \textbf{vani} $\downarrow$ & \textbf{adv} $\downarrow$ & \textbf{adv} $\downarrow$ & \textbf{adv} $\downarrow$ & \textbf{vani} $\uparrow$ & \textbf{vani} $\uparrow$ & \textbf{vani} $\uparrow$ \\
\textbf{Method} & \textbf{ASR} & \textbf{ASR} & \textbf{ASR} & \textbf{ASR} & \textbf{ASR} & \textbf{RTA} & \textbf{RTA} & \textbf{Comply} \\
\midrule
Qwen2.5-3B-IT & 0.282 & 0.119 & 0.949 & 0.677 & 0.241 & 0.610 & 0.497 & \textbf{0.960} \\
\quad + \textsc{Self-RedTeam} & 0.234 & 0.056 & 0.885 & 0.603 & 0.156 & 0.655 & 0.559 & 0.932 \\
\quad + \textsc{LoRA r=16 (step 50)} & 0.240 & 0.097 & 0.947 & 0.630 & 0.197 & 0.600 & 0.518 & 0.944 \\
\quad + \textsc{LoRA r=16 (step 100)} & 0.196 & 0.056 & 0.867 & 0.520 & 0.113 & 0.605 & 0.594 & 0.944 \\
\quad + \textsc{LoRA r=32 (step 50)} & 0.240 & 0.073 & 0.913 & 0.573 & 0.147 & 0.645 & 0.552 & 0.952 \\
\quad + \textsc{LoRA r=32 (step 100)} & \textbf{0.122} & \textbf{0.019} & \textbf{0.721} & \textbf{0.320} & \textbf{0.088} & \textbf{0.680} & \textbf{0.669} & 0.952 \\
\midrule
\it Improv. \textsc{Self-RedTeam} (\%) & \color{myblue}{+17.0} & \color{myblue}{+52.9} & \color{myblue}{+6.7} & \color{myblue}{+10.9} & \color{myblue}{+35.3} & \color{myblue}{+7.4} & \color{myblue}{+12.5} & \color{myred}{-2.9} \\
\it Improv. LoRA r=16 (50) (\%) & \color{myblue}{+14.9} & \color{myblue}{+18.5} & \color{myblue}{+0.2} & \color{myblue}{+6.9} & \color{myblue}{+18.3} & \color{myred}{-1.6} & \color{myblue}{+4.2} & \color{myred}{-1.7} \\
\it Improv. LoRA r=16 (100) (\%) & \color{myblue}{+30.5} & \color{myblue}{+52.9} & \color{myblue}{+8.6} & \color{myblue}{+23.2} & \color{myblue}{+53.1} & \color{myred}{-0.8} & \color{myblue}{+19.5} & \color{myred}{-1.7} \\
\it Improv. LoRA r=32 (50) (\%) & \color{myblue}{+14.9} & \color{myblue}{+38.7} & \color{myblue}{+3.8} & \color{myblue}{+15.4} & \color{myblue}{+39.0} & \color{myblue}{+5.7} & \color{myblue}{+11.1} & \color{myred}{-0.8} \\
\it Improv. LoRA r=32 (100) (\%) & \color{myblue}{+56.7} & \color{myblue}{+84.0} & \color{myblue}{+24.0} & \color{myblue}{+52.7} & \color{myblue}{+63.5} & \color{myblue}{+11.5} & \color{myblue}{+34.6} & \color{myred}{-0.8} \\
\bottomrule
\end{tabular}
}
\caption{Comparative performance of Qwen2.5-3B-IT, \textsc{Self-RedTeam} fine-tuned, and ABS models on six benchmarks for harmful refusal, two for benign compliance, and five (average) for instruction following and reasoning. Improvement is calculated relative to the base model. {\color{myblue}{Blue}} indicates a shift in the desired direction (lower ASR or higher RTA/Compliance), while {\color{myred}{red}} indicates regression.}
\label{tab:comparative_performance_3b}
\end{table*}

\begin{table*}[ht]
\centering
\resizebox{\textwidth}{!}{
\begin{tabular}{l | cc | c | c | c | c || c | c}
\toprule
& \multicolumn{6}{c||}{\textbf{Harmful Refusal}} & \multicolumn{2}{c}{\textbf{Benign Compliance}} \\
\cline{2-9}
& \multicolumn{2}{c|}{\textbf{WG:Test}} & \textbf{WJB} & \textbf{DAN} & \textbf{HarmBench} & \textbf{XSTest} & \textbf{StrongREJECT} & \textbf{XSTest} \\
& \textbf{adv} $\downarrow$ & \textbf{vani} $\downarrow$ & \textbf{adv} $\downarrow$ & \textbf{adv} $\downarrow$ & \textbf{adv} $\downarrow$ & \textbf{vani} $\uparrow$ & \textbf{vani} $\uparrow$ & \textbf{vani} $\uparrow$ \\
\textbf{Method} & \textbf{ASR} & \textbf{ASR} & \textbf{ASR} & \textbf{ASR} & \textbf{ASR} & \textbf{RTA} & \textbf{RTA} & \textbf{Comply} \\
\midrule
Qwen2.5-7B-IT & 0.279 & 0.040 & 0.897 & 0.613 & 0.234 & 0.915 & 0.470 & 0.912 \\
\quad + \textsc{Self-RedTeam} & 0.178 & 0.002 & 0.726 & 0.423 & 0.156 & 0.925 & 0.371 & 0.888 \\
\quad + \textsc{LoRA r=16 (step 50)} & 0.285 & 0.024 & 0.842 & 0.503 & 0.219 & 0.905 & 0.563 & 0.896 \\
\quad + \textsc{LoRA r=16 (step 100)} & 0.169 & 0.010 & 0.756 & 0.383 & 0.156 & 0.915 & 0.644 & 0.884 \\
\quad + \textsc{LoRA r=32 (step 50)} & 0.217 & 0.022 & 0.802 & 0.437 & 0.194 & 0.915 & 0.619 & 0.904 \\
\quad + \textsc{LoRA r=32 (step 100)} & \textbf{0.807} & \textbf{0.808} & \textbf{1.000} & \textbf{1.000} & \textbf{0.984} & \textbf{0.220} & \textbf{0.017} & \textbf{0.032} \\
\midrule
\it Improv. \textsc{Self-RedTeam} (\%) & \color{myblue}{+36.18} & \color{myblue}{+94.19} & \color{myblue}{+19.12} & \color{myblue}{+30.98} & \color{myblue}{+33.36} & \color{myblue}{+1.09} & \color{myred}{-21.24} & \color{myred}{-2.63} \\
\it Improv. LoRA r=16 (50) (\%) & \color{myred}{-2.15} & \color{myblue}{+41.16} & \color{myblue}{+6.19} & \color{myblue}{+17.94} & \color{myblue}{+6.66} & \color{myred}{-1.09} & \color{myblue}{+19.73} & \color{myred}{-1.75} \\
\it Improv. LoRA r=16 (100) (\%) & \color{myblue}{+39.37} & \color{myblue}{+76.51} & \color{myblue}{+15.77} & \color{myblue}{+37.50} & \color{myblue}{+33.36} & 0.00 & \color{myblue}{+36.97} & \color{myred}{-3.07} \\
\it Improv. LoRA r=32 (50) (\%) & \color{myblue}{+22.34} & \color{myblue}{+47.22} & \color{myblue}{+10.65} & \color{myblue}{+28.80} & \color{myblue}{+17.36} & 0.00 & \color{myblue}{+31.68} & \color{myred}{-0.88} \\
\it Improv. LoRA r=32 (100) (\%) & \color{myred}{-189.39} & \color{myred}{-1857.14} & \color{myred}{-11.43} & \color{myred}{-63.05} & \color{myred}{-319.97} & \color{myred}{-75.96} & \color{myred}{-96.47} & \color{myred}{-96.49} \\
\bottomrule
\end{tabular}
}
\caption{Comparative performance of Qwen2.5-7B-IT, \textsc{Self-RedTeam} fine-tuned, and ABS models on six benchmarks for harmful refusal, two for benign compliance, and five (average) for instruction following and reasoning. Improvement is calculated relative to the base model. {\color{myblue}{Blue}} indicates a shift in the desired direction (lower ASR or higher RTA/Compliance), while {\color{myred}{red}} indicates regression. Notice how the last ABS training collapsed and the performance reduced to random guess over all the benchmarks.}
\label{tab:comparative_performance_7b}
\end{table*}

\begin{table*}[ht]
\centering
\resizebox{\textwidth}{!}{
\begin{tabular}{l | cc | c | c | c | c || c | c}
\toprule
& \multicolumn{6}{c||}{\textbf{Harmful Refusal}} & \multicolumn{2}{c}{\textbf{Benign Compliance}} \\
\cline{2-9}
& \multicolumn{2}{c|}{\textbf{WG:Test}} & \textbf{WJB} & \textbf{DAN} & \textbf{HarmBench} & \textbf{XSTest} & \textbf{StrongREJECT} & \textbf{XSTest} \\
& \textbf{adv} $\downarrow$ & \textbf{vani} $\downarrow$ & \textbf{adv} $\downarrow$ & \textbf{adv} $\downarrow$ & \textbf{adv} $\downarrow$ & \textbf{vani} $\uparrow$ & \textbf{vani} $\uparrow$ & \textbf{vani} $\uparrow$ \\
\textbf{Method} & \textbf{ASR} & \textbf{ASR} & \textbf{ASR} & \textbf{ASR} & \textbf{ASR} & \textbf{RTA} & \textbf{RTA} & \textbf{Comply} \\
\midrule
Qwen2.5-14B-IT & 0.146 & 0.017 & 0.845 & 0.647 & 0.119 & 0.735 & 0.647 & 0.984 \\
\quad + \textsc{Self-RedTeam} & 0.116 & 0.005 & 0.661 & 0.467 & 0.056 & 0.775 & 0.701 & 0.968 \\
\quad + \textsc{LoRA r=32 (step 50)} & 0.137 & 0.020 & 0.828 & 0.600 & 0.115 & 0.770 & 0.654 & 0.988 \\
\quad + \textsc{LoRA r=32 (step 100)} & 0.119 & 0.015 & 0.779 & 0.487 & 0.084 & 0.750 & 0.725 & 0.980 \\
\quad + \textsc{LoRA r=32 (step 125)} & 0.131 & 0.022 & 0.706 & 0.446 & 0.093 & 0.770 & 0.748 & 0.968 \\
\midrule
\it Improv. \textsc{Self-RedTeam} (\%) 
& \color{myblue}{20.55} 
& \color{myblue}{70.59} 
& \color{myblue}{21.77} 
& \color{myblue}{27.82} 
& \color{myblue}{52.94} 
& \color{myblue}{5.44} 
& \color{myblue}{8.35} 
& \color{myred}{-1.63} \\

\it Improv. LoRA r=32 (50) (\%) 
& \color{myblue}{6.16} 
& \color{myred}{-17.65} 
& \color{myblue}{2.01} 
& \color{myblue}{7.26} 
& \color{myblue}{3.36} 
& \color{myblue}{4.76} 
& \color{myblue}{1.08} 
& \color{myblue}{0.41} \\

\it Improv. LoRA r=32 (100) (\%) 
& \color{myblue}{18.49} 
& \color{myblue}{11.76} 
& \color{myblue}{7.81}  
& \color{myblue}{24.73} 
& \color{myblue}{29.41} 
& \color{myblue}{2.00}  
& \color{myblue}{12.05} 
& \color{myred}{0.4} \\

\it Improv. LoRA r=32 (125) (\%)
& \color{myblue}{10.27}
& \color{myred}{-29.41}
& \color{myblue}{16.44}
& \color{myblue}{31.07}
& \color{myblue}{21.85}
& \color{myblue}{4.76}
& \color{myblue}{15.61}
& \color{myred}{-1.63} \\
\bottomrule
\end{tabular}
}
\caption{Comparative performance of Qwen2.5-14B-IT, \textsc{Self-RedTeam} fine-tuned, and ABS models on six benchmarks for harmful refusal, two for benign compliance, and five (average) for instruction following and reasoning. Improvement is calculated relative to the base model. {\color{myblue}{Blue}} indicates a shift in the desired direction (lower ASR or higher RTA/Compliance), while {\color{myred}{red}} indicates regression.}
\label{tab:comparative_performance_14b}
\end{table*}

\begin{table}[!ht]
    \centering
    \begin{tabular}{|c|c|c|c|c|c|}
        \hline  
         & GSM8K & Common Sense & BoolQ & ARC Ch. & HellaSwag\\
         \hline
        Qwen2.5-3B &  0.816 & 0.699 & 0.753 & \textbf{0.752} & 0.706\\
        \hline
        \textsc{Self-RedTeam} Qwen2.5-3B & 0.796 & 0.503$^*$ & 0.69$^*$ & 0.668 & 0.7 \\
        \hline
        Safe-Qwen2.5-3B & \textbf{0.83} & \textbf{0.7} & 0.716 & 0.729 & \textbf{0.713} \\
        \hline
        \hline
        \hline 
        Qwen2.5-7B &  \textbf{0.873} & \textbf{0.86} & \textbf{0.83} & \textbf{0.876} & 0.81\\
        \hline
        \textsc{Self-RedTeam} Qwen2.5-7B &  0.856 & 0.75$^*$ & \textbf{0.83} & 0.842 & \textbf{0.813} \\
        \hline
        LoRA Qwen2.5-7B & \textbf{0.873} & \textbf{0.86} & 0.823 & 0.866 & \textbf{0.813} \\
        \hline
        \hline
        \hline
        Qwen2.5-14B &  \textbf{0.91} & \textbf{0.85} & 0.836 & 0.906 & 0.87\\
        \hline
        \textsc{Self-RedTeam} Qwen2.5-14B & \textbf{0.91} & 0.836 & \textbf{0.843} & \textbf{0.923} & \textbf{0.88} \\
        \hline
        LoRA Qwen2.5-14B & \textbf{0.91} & \textbf{0.85} & 0.836 & 0.906 & 0.87 \\
        \hline
    \end{tabular}
    \caption{Comparison of the performance of Qwen-3B, 7B, and 14B with the corresponding \textsc{Self-RedTeam} and LoRA versions on common reasoning benchmarks. Best results across model size are in \textbf{bold}; results with * are statistically lower than the worst performing model (McNemar + Bonferroni correction with p-value at $0.05$). Runs are on 300 samples.}
    \label{tab:performance-safety-qwen}
\end{table}

\end{document}